\documentclass{article}

\usepackage{tgpagella}
\usepackage[utf8]{inputenc} 
\usepackage[T1]{fontenc}    
\usepackage{hyperref}       
\usepackage{url}            
\usepackage{booktabs}       
\usepackage{nicefrac}       
\usepackage{microtype}      
\usepackage{authblk}
\usepackage{multirow}

\usepackage{graphicx,amsmath,amsfonts,amscd,amssymb,bm,url,color,latexsym}

\usepackage[inline]{enumitem}
\usepackage{mwe}
\usepackage{subcaption}
\usepackage{algorithm,algorithmicx}
\usepackage{algpseudocode}
\usepackage[nameinlink]{cleveref}
\usepackage{framed}
\usepackage{comment}
\usepackage{wrapfig}
\usepackage{tikz}
\usepackage{pgfplots}

\usepackage{cite}

\usepackage{tablefootnote}
\usepackage{tcolorbox}
\usepackage{pifont}
\usepackage{authblk}
\usepackage[margin=0.95in]{geometry}

\hypersetup{
    colorlinks=true,%
    citecolor=blue,%
    filecolor=blue,%
    linkcolor=blue,%
    urlcolor=blue
}
\usepackage[toc, page]{appendix}

\newtheorem{theorem}{Theorem}[section]

\newtheorem{proposition}[theorem]{Proposition}

\renewcommand{\mathbf}{\boldsymbol}

\newcommand{\conv}{\circledast}

\newcommand{\mb}{\mathbf}
\newcommand{\mc}{\mathcal}

\newcommand{\bb}{\mathbb}

\newcommand{\eps}{\varepsilon}

\newcommand{ \Brac }[1]{\left\lbrace #1 \right\rbrace}
\newcommand{ \brac }[1]{\left[ #1 \right]}
\newcommand{ \paren }[1]{ \left( #1 \right) }

\DeclareMathOperator{\diag}{diag}
\DeclareMathOperator{\sign}{sign}

\newcommand{\wh}{\widehat}
\newcommand{\wt}{\widetilde}
\newcommand{\ol}{\overline}

\newcommand{\norm}[2]{\left\| #1 \right\|_{#2}}
\newcommand{\abs}[1]{\left| #1 \right|}

\def \endprf{\hfill {\vrule height6pt width6pt depth0pt}\medskip}
\newenvironment{proof}{\noindent {\bf Proof} }{\endprf\par}

\newcommand{\shift}[2]{{\mathrm{s}_{#2}}\left[#1\right]}

\newcommand{\Req}{\textbf{Require:}\hspace*{0.5em}}

\newcommand{\X}{\hspace*{3mm}}

\newcommand{\cm}[1]{$\triangleright$ #1}

\usepackage{graphicx}
\usepackage{booktabs,tabularx,graphicx}
\usepackage{amssymb}
\usepackage{pifont}
\newcommand{\cmark}{\ding{51}}%
\newcommand{\xmark}{\ding{55}}%

\usepackage{algorithm,lipsum,changepage}
\usepackage{algpseudocode}

\title{Convolutional Normalization: Improving Deep Convolutional Network Robustness and Training}

\author[$\sharp$]{Sheng Liu\thanks{The first two authors contributed to this work equally.}$^,$}
\author[$\sharp$, $\S$]{Xiao Li}
\author[$\dagger$]{Yuexiang Zhai}
\author[$\dagger$]{Chong You}
\author[$\natural$]{\\Zhihui Zhu}
\author[$\sharp$, $\diamondsuit$]{Carlos Fernandez-Granda}
\author[$\S$]{Qing Qu}

\affil[$\sharp$]{Center for Data Science, New York University}
\affil[$\diamondsuit$]{Courant Institute of Mathematical Sciences, New York University}
\affil[$\dagger$]{Department of EECS, UC Berkeley}
\affil[$\natural$]{Electrical and Computer Engineering, University of Denver}
\affil[$\S$]{Department of EECS, University of Michigan}

\begin{document}

\maketitle

\begin{abstract}
Normalization techniques have become a basic component in modern convolutional neural networks (ConvNets). In particular, many recent works demonstrate that promoting the orthogonality of the weights helps train deep models and improve robustness. For ConvNets, most existing methods are based on penalizing or normalizing weight matrices derived from concatenating or flattening the convolutional kernels. These methods often destroy or ignore the benign convolutional structure of the kernels; therefore, they are often expensive or impractical for deep ConvNets. In contrast, we introduce a simple and efficient ``Convolutional Normalization'' (ConvNorm) method that can fully exploit the convolutional structure in the Fourier domain and serve as a simple plug-and-play module to be conveniently incorporated into any ConvNets. Our method is inspired by recent work on preconditioning methods for convolutional sparse coding and can effectively promote each layer's channel-wise isometry. Furthermore, we show that our ConvNorm can reduce the layerwise spectral norm of the weight matrices and hence improve the Lipschitzness of the network, leading to easier training and improved robustness for deep ConvNets. Applied to classification under noise corruptions and generative adversarial network (GAN), we show that the ConvNorm improves the robustness of common ConvNets such as ResNet and the performance of GAN. We verify our findings via numerical experiments on CIFAR and ImageNet.
\end{abstract}

\section{Introduction}\label{sec:intro}
In the past decade, Convolutional Neural Networks (ConvNets) have achieved phenomenal success in many machine learning and computer vision applications \cite{krizhevsky2012imagenet,simonyan2014very,szegedy2015going,ioffe2015batch,he2016deep,xie2017aggregated,huang2017densely}. Normalization is one of the most important components of modern network architectures \cite{huang2020normalization}. Early normalization techniques, such as batch normalization (BatchNorm) \cite{ioffe2015batch}, are cornerstones for effective training of models beyond a few layers. Since then, the values of normalization for optimization and learning is extensively studied, and many normalization techniques, such as layer normalization \cite{ba2016layer}, instance normalization \cite{ulyanov2016instance}, and group normalization \cite{wu2018group} are proposed. Many of such normalization techniques are based on estimating certain statistics of neuron inputs from training data. However, precise estimations of the statistics may not always be possible. For example, BatchNorm becomes ineffective when the batch size is small \cite{he2017mask}, or batch samples are statistically dependent \cite{sun2020test}.

Weight normalization \cite{salimans2016weight} is a powerful alternative to BatchNorm that improves the conditioning of neural network training without the need to estimate statistics from neuron inputs. Weight normalization operates by either reparameterizing or regularizing the network weights so that all the weights have unit Euclidean norm.  Since then, various forms of normalization for network weights are proposed and become critical for many tasks such as training Generative Adversarial Networks (GANs) \cite{miyato2018spectral} and obtaining robustness to input perturbations \cite{araujo2021lipschitz, qian2018l2}. One of the most popular forms of weight normalization is enforcing orthogonality, which has drawn attention from a diverse range of research topics. The idea is that weights in each layer should be orthogonal and energy-preserving. Orthogonality is argued to play a central role for training ultra-deep models \cite{saxe2013exact,pennington2018emergence,xiao2018dynamical,hu2020provable,qi2020deep}, optimizing recurrent models \cite{arjovsky2016unitary,vorontsov2017orthogonality,helfrich2018orthogonal,lezcano2019cheap}, improving generalization \cite{jia2019orthogonal}, obtaining robustness \cite{cisse2017parseval,trockman2021orthogonalizing}, learning disentangled features \cite{liu2020oogan,ye2020network}, improving the quality of GANs \cite{brock2018large,odena2018generator}, learning low-dimensional embedding \cite{atzmon2020isometric}, etc. 

\begin{figure}[t]
	\centering
    \includegraphics[width = 0.75\linewidth]{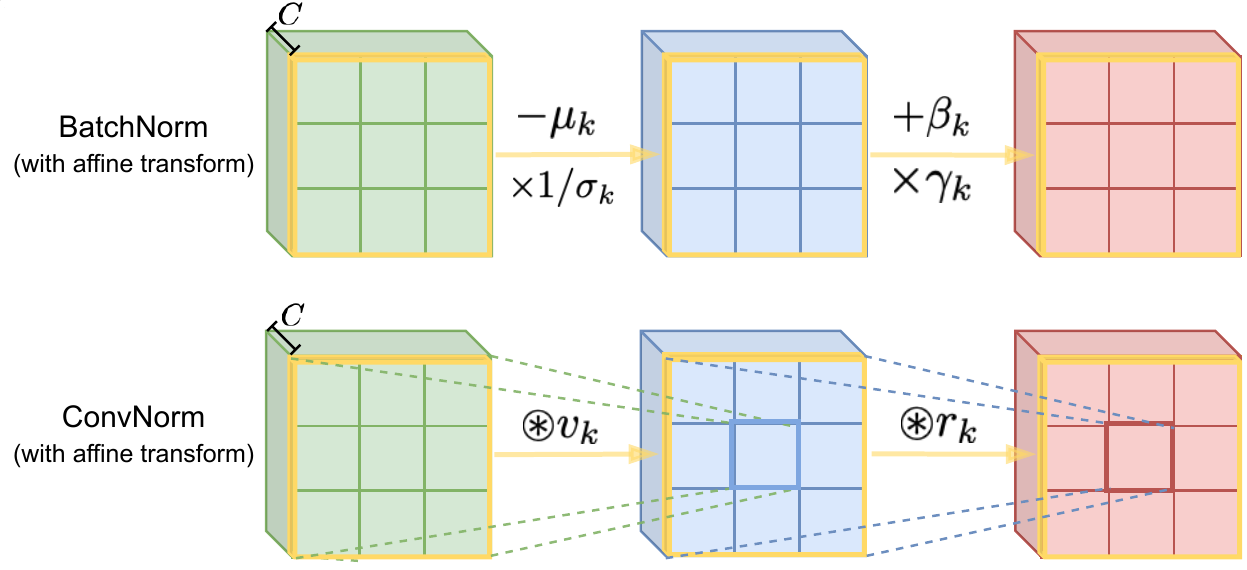}
	\caption{\textbf{Comparison between BatchNorm and ConvNorm} on activations of $k=1,\dots,C$ channels. BatchNorm subtracts and multiplies the activations of each channel by computed scalars: mean $\mu$ and variance $\sigma^2$, before a per-channel affine transform parameterized by learned parameters $\beta$ and $\gamma$; ConvNorm performs per-channel convolution with precomputed kernel $v$ to normalize the spectrum of the weight matrix for the convolution layer, following with a channel-wise convolution with learned kernel $r$ as the affine transform.}
	\label{fig:ConvNorm}
\end{figure}

\paragraph{Exploiting convolution structures for normalization.} 
Our work is motivated by the pivotal role of weight normalization in deep learning. In the context of ConvNets, the network weights are multi-dimensional (e.g., 4-dimensional for a 2D ConvNet) convolutional kernels. A vast majority of existing literature \cite{harandi2016generalized,cisse2017parseval,bansal2018can,zhang2019approximated,jia2019orthogonal,Li2020Efficient,huang2020controllable} imposes orthogonal weight regularization for ConvNets by treating multi-dimensional convolutional kernels as 2D matrices (e.g., by flattening certain dimensions) and imposing orthogonality of the matrix. However, this choice ignores the translation-invariance properties of convolutional operators and, as shown in \cite{qi2020deep}, does not guarantee energy preservation. On the other hand, these methods often involve dealing with matrix inversions that are computationally expensive for deep and highly overparameterized networks. 

In contrast, in this work we introduce a new normalization method dedicated to ConvNets, which explicitly exploits translation-invariance properties of convolutional operators. Therefore, we term our method as \emph{Convolutional Normalization} (ConvNorm). We normalize each output channel for each layer of ConvNets, similar to recent preconditioning methods for convolutional sparse coding \cite{qu2020geometric}. The ConvNorm can be viewed as a reparameterization approach for the kernels, that actually it normalizes the weight of each channel to be tight frame.\footnote{Tight frame can be viewed as a generalization of orthogonality for overcomplete matrices, which is also energy preserving.}  While extra mathematical hassles do exist in incorporating translation-invariance properties, and it turns out to be a blessing, rather than a curse, in terms of computation, as it allows us to carry out the inversion operation in our ConvNorm via fast Fourier transform (FFT) in the frequency domain, for which the computation complexity can be significantly reduced.

\paragraph{Highlights of our method.} 
In summary, for ConvNets our approach enjoys several clear advantages over classical normalization methods \cite{huang2017orthogonal, li2019preventing, wang2019orthogonal}, that we list below:

\begin{itemize}[leftmargin=*]
    
    \item \textbf{Easy to implement.} In contrast to weight regularization methods that often require hyperparameter tuning and heavy computation \cite{huang2017orthogonal,wang2019orthogonal}, the ConvNorm has no parameter to tune and is efficient to compute. Moreover, the ConvNorm can serve as a simple plug-and-play module that can be conveniently incorporated into training almost any ConvNets. 

    \item \textbf{Improving network robustness.} Although the ConvNorm operates on each output channel separately, we show that it actually improves the overall layer-wise Lipschitzness of the ConvNets. Therefore, as demonstrated by our experiments, it has superior robustness performance against noise corruptions and adversarial attacks.

    \item \textbf{Improving network training.} We numerically demonstrate that the ConvNorm accelerates training on standard image datasets such as CIFAR~\cite{krizhevsky2009learning} and ImageNet~\cite{russakovsky2015imagenet}. Inspired by the work \cite{qu2019nonconvex,qu2020geometric}, our high-level intuition is that the ConvNorm improves the optimization landscape that optimization algorithms converge faster to the desired solutions.
        
\end{itemize}

\paragraph{Related work.} Besides our work, a few very recent work also exploits the translation-invariance for designing the normalization techniques of ConvNets. We summarize and explain the difference with our method below.

\begin{itemize}[leftmargin=*]
    \item The work \cite{wang2019orthogonal,qi2020deep} derived a similar notion of orthogonality for convolutional kernels, and adopted a penalty based method to enforce orthogonality for network weights. These penalty methods often require careful tuning of the strength of the penalty on a case-by-case basis. In contrast, our method is parameter-free and thus easier to use. Our method also shows better empirical performance in terms of robustness.

    \item Very recent work by \cite{trockman2021orthogonalizing} presented a method to enforce \emph{strict} orthogonality of convolutional weights by using Cayley transform. Like our approach, a sub-step of their method utilizes the idea of performing the computation in the Fourier domain. However, as they normalize the whole unstructured weight matrix, computing expensive matrix inversion is inevitable, so that their running time and memory consumption is prohibitive for large networks.\footnote{In \cite{trockman2021orthogonalizing}, the results are reported based on ResNet9, whereas our method can be easily added to larger networks, e.g. ResNet18 and ResNet50.} In contrast, our method is ``orthogonalizing'' the weight of each channel instead of the whole layer, so that we can exploit the convolutional structure to avoid expensive matrix inversion with a much lower computational burden. In the meanwhile, we show that this channel-wise normalization can still improve layer-wise Lipschitz condition.
\end{itemize}

\paragraph{Organizations.} The rest of our paper is organized as follows. In \Cref{sec:cnn_review}, we introduce the basic notations and provide a brief overview of ConvNets. In \Cref{sec:iso-bn}, we introduce the design of the proposed ConvNorm and discuss the key intuitions and advantages. In \Cref{sec:exp}, we perform extensive experiments on various applications verifying the effectiveness of the proposed method. Finally, we conclude and point to some interesting future directions in \Cref{sec:conclusion}. To streamline our presentation, some technical details are deferred to the Appendices.
The code of implementing our ConvNorm can be found online:
\begin{center}
    \url{https://github.com/shengliu66/ConvNorm}.
\end{center}

\section{Preliminary}\label{sec:cnn_review}

\paragraph{Review of deep networks.} A deep network is essentially a \emph{nonlinear} mapping $f(\cdot): \mb x \mapsto \mb y$, which can be modeled by a composition of a series of simple maps: $f(\mb x) = f^{L-1} \circ \cdots \circ f^1\circ f^0(\mb x)$, where every $f^\ell(\cdot)\;(1\leq \ell \leq L)$ is called one ``layer''. Each layer is composed of a linear transform, followed by a simple nonlinear activation function $\varphi(\cdot)$.\footnote{The nonlinearity could contain BatchNorm \cite{ioffe2015batch}, pooling, dropout \cite{srivastava2014dropout}, and stride, etc.} More precisely, a basic deep network of $L$ layers can be defined recursively by interleaving linear and nonlinear activation layers as

\begin{align}\label{eqn:deep-network-layer}
    \mb z^{\ell+1} = f^\ell(\mb z^\ell) = \varphi \circ \mc A^\ell (\mb z^\ell) 
\end{align}
for $\ell = 0,1,\ldots,L-1,$ with $\mb z_0 = \mb x$. Here $\mc A^\ell(\cdot)$ denotes the linear transform and will be described in detail soon. For convenience, let us use $\mb \theta$ to denote all network parameters in $\Brac{\mc A^\ell(\cdot) }_{\ell=0}^{L-1}$. The goal of deep learning is to fit the observation $\mb y$ with the output $f(\mb x, \mb \theta)$ for any sample $\mb x$ from a distribution $\mc D$, by learning $\mb \theta$. This can be achieved by optimizing a certain loss function $\ell(\cdot)$, i.e., 

\begin{align*}
     \min_{\mb \theta \in \mb \Theta} \;L(\mb \theta; \Brac{\paren{\mb x^i,\mb y^i} }_{i=1}^m) \;:=\; \frac{1}{m} \sum_{i=1}^m \ell\paren{ f(\mb x^i,\mb \theta), \mb y^i },
\end{align*}
given a (large) training dataset $\Brac{\paren{\mb x^i,\mb y^i} }_{i=1}^m$. For example, for a typical classification task, the class label of a sample $\mb x$ is represented by a one-hot vector $\mb y \in \mathbb{R}^k$ representing its membership in $k$ classes. The loss can be chosen to be either the cross-entropy or $\ell_2$-loss \cite{janocha2017loss}. In the following, we use $(\mb x,\mb y)$ to present one training sample.


\paragraph{An overview of ConvNets.} The ConvNet \cite{lecun1998gradient} is a special deep network architecture, where each of its linear layer can be implemented much more efficiently via convolutions in comparison to fully connected networks \cite{lecun2015deep}. Because of its efficiency and popularity in machine learning, for the rest of the paper, we focus on ConvNets. Suppose the input data $\mb x$ has $C$ channels, represented as

\begin{align}\label{eqn:input-X}
    \mb x \;=\; \paren{ \mb x_{1},  \mb x_{2},\cdots, \mb x_{C}  },
\end{align}
where for 1D signal $\mb x_{k} \in \bb R^m $ denotes the $k$th channel feature of $\mb x$.\footnote{If the data is 2D, we can assume $\mb x \in \bb R^{m_1 \times m_2}$. For simplicity, we present our idea based on 1D signal.} For the $\ell$th layer $(0\leq \ell \leq L-1)$ of ConvNets, the linear operator $\mc A^\ell(\cdot): \bb R^{C_{\ell} \times m} \mapsto \bb R^{ C_{\ell+1} \times m }$ in \eqref{eqn:deep-network-layer} is a convolution operation with $C_{\ell+1}$ output channels,

\begin{align*}
    \mb z^{\ell+1} \;&=\; \paren{ \mb z_{1}^{\ell+1}, \mb z_{2}^{\ell+1}, \cdots, \mb z_{C_{\ell+1}}^{\ell+1} }, \\
     \mb z_{k}^{\ell+1} \;&=\; \varphi\paren{ \sum_{j=1}^{C_{\ell}} \mb a_{kj}^\ell \ast \mb z_{j}^\ell } \quad (1\leq k \leq C_{\ell+1} ),
\end{align*}
where $\ast$ denotes the convolution between two items that we will discuss below in more detail. Thus, for the $\ell$th layer with $C_{\ell}$ input channels and $C_{\ell+1}$ output channels, we can organize the convolution kernels $\Brac{\mb a_{kj} }$ as

\begin{align*}
    \mb A^\ell \;=\; \begin{bmatrix}
      \mb a_{11}^\ell & \mb a_{12}^\ell & \cdots & \mb a_{1C_{\ell}}^\ell \\
    \mb a_{21}^\ell & \mb a_{22}^\ell &  \cdots & \mb a_{2 C_{\ell} }^\ell \\
    \vdots  & \vdots  & \ddots & \vdots \\
     \mb a_{C_{\ell+1} 1 }^\ell &  \mb a_{C_{\ell+1} 2 }^\ell &  \cdots  & \mb a_{C_{\ell+1} C_{\ell}  }^\ell
    \end{bmatrix}. 
\end{align*}

\paragraph{Convolution operators.} For the simplicity of presentation and analysis, we adopt \emph{circular} convolution instead of linear convolution.\footnote{Although there are slight differences between linear and circulant convolutions on the boundaries, actually any linear convolution can be \emph{reduced} to circular convolution simply via zero-padding.} For 1D signal, given a kernel $\mb a \in \bb R^n $ and an input signal $\mb x \in \bb R^m$ (in many cases $m\gg n$), a circular convolution $\ast$ between $\mb a$ and $\mb x$ can be written in a simple matrix-vector product form via

\begin{align*}
    \mb y \;=\; \mb a \ast \mb x \;=\; \mb C_{ \mb a } \cdot  \mb x,
\end{align*}
where $\mb C_{\mb a}$ denotes a circulant matrix of (zero-padded) $\mb a$,

\begin{align*}
    \mb C_{\mb a} \;:=\; \begin{bmatrix}
    \shift{\mb a}{0} & \shift{\mb a}{1} & \cdots & \shift{\mb a}{m-1}
    \end{bmatrix},
\end{align*}
which is the concatenation of all cyclic shifts $\shift{\mb a}{k}\;(0\leq k \leq m-1)$ of length $k$ of the (zero-padded) vector $\mb a$. Since $\mb C_{ \mb a }$ can be decomposed via the discrete Fourier transform (DFT) matrix $\mb F$:

\begin{align}\label{eqn:circ-decomp}
    \mb C_{\mb a} \;=\; \mb F^* \diag\paren{ \wh{\mb a} } \mb F, \quad \wh{\mb a}\;=\; \mb F\mb a,
\end{align}
where $\wh{\mb a}$ denotes the Fourier transform of a vector $\mb a$. The computation of $\mb a \ast \mb x$ can be carried out efficiently via fast Fourier transform (FFT) in the frequency domain. We refer the readers to the appendix for more technical details.

\section{Convolutional Normalization}\label{sec:iso-bn}
In the following, we introduce the proposed ConvNorm, that can fully exploit benign convolution structures of ConvNets. It can be efficiently implemented in the frequency domain, and reduce the layer-wise Lipschitz constant. First of all, we build intuitions of the new design from the simplest setting. From this, we show how to expand the idea to practical ConvNets and discuss its advantages for training and robustness.

\subsection{A warm-up study}\label{subsec:warm-up}
\vspace{-0.05in}
Let us build some intuitions by zooming into one layer of ConvNets with both input and output being \emph{single-channel}, 
\vspace{-0.05in}
\begin{align}\label{eqn:single-kernel-l-layer}
    \mb z_{out} \;= \; \mc A_L (\mb z) \;=\;  \mb a  \ast \mb z_{in}, 
\end{align}
where $\mb z_{in}$ is the input signal, $\mb a $ is a single kernel, and $\mb z_{out}$ denotes the output before the nonlinear activation. The form \eqref{eqn:single-kernel-l-layer} is closely related to recent work on blind deconvolution \cite{qu2019nonconvex}. More specifically, the work showed that normalizing the output $\mb z_{out}$ via \emph{preconditioning} eliminates bad local minimizers and dramatically improves the optimization landscapes for learning the kernel $\mb a$. The basic idea is to multiply a preconditioning matrix which approximates the following form\footnote{In the work \cite{qu2019nonconvex}, they cook up a matrix by using output samples $\wt{\mb P} = \paren{ \frac{C}{m} \sum_{i=1}^m \mb C_{\mb z_{out}^i} (\mb C_{\mb z_{out}^i})^\top   }^{-1/2}$. When the input samples $\mb z_{in}^i$ are i.i.d. zero mean, it can be showed that $\wt{\mb P} \approx \mb P$ for large $m$. For ConvNets, we can just use the learned kernel $\mb a$ for cooking up $\mb P$.}
\vspace{-0.05in}
\begin{align}
    \mb P \;=\; \paren{ \mb C_{\mb a} \mb C_{\mb a}^\top  }^{-1/2}.
\label{eq:p-matrix}
\end{align}
\vspace{-0.05in}
As we observe
\vspace{-0.02in}
\begin{align*}
    \wt{\mb z}_{out}\;=\; \mb P \mb z_{out} \;=\; \underbrace{\paren{ \mb C_{\mb a} \mb C_{\mb a}^\top  }^{-1/2} \mb C_{\mb a} }_{\mb Q(\mb a)} \cdot \mb z_{in},
\end{align*}
the ConvNorm is essentially \emph{reparametrizing} the circulant matrix $\mb C_{\mb a}$ of the kernel $\mb a$ to an \emph{orthogonal} circulant matrix $\mb Q(\mb a)= \paren{ \mb C_{\mb a} \mb C_{\mb a}^\top  }^{-1/2} \mb C_{\mb a}$, with $\mb Q \mb Q^\top = \mb I$. Thus, the ConvNorm is improving the conditioning of the \emph{vanilla} problem and reducing the Lipschitz constant of the operator $\mc A_L(\cdot)$ in \eqref{eqn:single-kernel-l-layer}. On the other hand, the benefits of this normalization can also be observed in the frequency domain. Based on \eqref{eqn:circ-decomp}, we have $ \mb P = \mb F^* \diag\paren{ \mb v } \mb F = \mb C_{\mb v} $ with $\mb v = \mb F^{-1}\paren{ \abs{\wh{\mb a}}^{\odot-1} }$. Thus, we also have
\vspace{-0.05in}
\begin{align*}
    \mb Q(\mb a) = \mb C_{\mb v} \cdot \mb C_{\mb a}  \;=\; \mb C_{\mb v \ast \mb a} = \mb F^* \diag( \wh{ \mb g}(\mb a)  ) \mb F, \quad  \wh{ \mb g}(\mb a)  \;=\; \wh{\mb a} \odot \abs{ \wh{\mb a} }^{\odot-1},
\end{align*}
with $\odot$ denoting entrywise operation and $\mb g = \mb F^{-1}\paren{\wh{\mb a} \odot \abs{ \wh{\mb a} }^{\odot-1}}$. Thus, we can see that:
\vspace{-0.12in}
\begin{itemize}[leftmargin=*]
    \item Although the reparameterization involves matrix inversion, which is typically expensive to compute, for convolution it can actually be much more efficiently implemented in the frequency domain via FFT, reducing the complexity from $O(n^3)$ to $O(n\log n)$.
    \vspace{-0.03in}
    \item The reparametrized kernel $\mb g$ is effectively an \emph{all-pass} filter with flat normalized spectrum $\wh{\mb a} \odot \abs{ \wh{\mb a} }^{\odot-1}$.\footnote{An all-pass filter is a signal processing filter that passes all frequencies equally in gain, but can change the phase relationship among various frequencies.}  From an information theory perspective, this implies that it can better preserve (in particular, high-frequency) information of the input feature from the previous layer. 
\end{itemize}

\subsection{ConvNorm for multiple channels}

So far, we only considered one layer ConvNets with single-channel input and output. However, recall from \Cref{sec:cnn_review}, modern deep ConvNets are usually designed with many layers; each typical layer is constructed with a linear transformation with \emph{multiple} input and output channels, followed by strides, normalization, and nonlinear activation. Extension of the normalization approach in \Cref{subsec:warm-up} from one layer to multiple layers is easy, which can be done by applying the same normalization repetitively for all the layers. However, generalizing our method from a single channel to multiple channels is not obvious, that we discuss below.

In \cite{qu2020geometric}, the work introduced a preconditioning method for normalizing multiple kernels in convolutional sparse coding. In the following, we show that such an idea can be adapted to normalize each output channel, reduce the Lipschitz constant of the weight matrix in each layer, and improve training and network robustness. Let us consider any layer $\ell$ ($1\leq \ell \leq L$) within a \emph{vanilla} ConvNet using $1$-stride, and take one channel (e.g., $k$-th channel) of that layer as an example. For simplicity of presentation, we hide the layer number $\ell$. Given $\mb z_{k,out} = \sum_{j=1}^{C_I}\mb a_{kj} \ast \mb z_{j,in}$, the $k$-th output channel can be written as
\vspace{-0.1in}
\begin{align*}
    \mb z_{k,out} \;=\; \underbrace{ \begin{bmatrix}
    \mb C_{\mb a_{k1}} & \mb C_{\mb a_{k2} } & \cdots &  \mb C_{\mb a_{kC_I } }
    \end{bmatrix} }_{ \mb A_k } \cdot \underbrace{ 
    \begin{bmatrix}
        {\mb z}_{1,in}  \\
        {\mb z}_{2,in} \\
       \vdots \\
        {\mb z}_{C_{I},in} 
    \end{bmatrix} }_{  \mb z_{in} },
\end{align*}
with $C_{I}$ and $C_{O}$ being the numbers of input and output channels, respectively. For each channel $k=1,\cdots, C_O$, we normalize the output by 
\begin{align}\label{eqn:Conv-Norm-k}
  \boxed{ \mb P_k \;=\; \paren{\sum_{j=1}^{C_I} \mb C_{ \mb a_{kj} } \mb C_{\mb a_{kj}}^\top }^{-1/2} = \paren{ \mb A_k \mb A_k^\top }^{-1/2},}
\end{align}
    \vspace{-2mm}
so that
\begin{align}\label{eqn:precond-multi}
    \wt{\mb z}_{k,out} \;=\; \mb P_k \mb z_{k,out} \;=\; \underbrace{ \paren{ \mb A_k \mb A_k^\top }^{-1/2} \mb A_k}_{\mb Q_k (\mb A_k) } \cdot \mb z_{in}.
\end{align}
Thus, we can see the ConvNorm is essentially a reparameterization of the kernels $\Brac{\mb a_{kj}}_{j=1}^{C_I}$ for the $k$-th channel. Similar to \Cref{subsec:warm-up}, the operation can be rewritten in the form of convolutions 
\begin{align*}
    \mb Q_k(\mb A_k) \;=\; \mb P_k \mb A_k \;=\; \begin{bmatrix}
        \mb C_{\mb v_k \ast \mb a_{k1} } & \cdots & \mb C_{\mb v_k \ast \mb a_{kC_I} }  
    \end{bmatrix}
\end{align*}
with $\mb P_k = \mb C_{\mb v_k}$ and $\mb v_k = \mb F^{-1} \paren{ \sum_{i=1}^{C_I}  \abs{ \wh{\mb a}_{ki} }^{\odot 2}   }^{\odot -1/2}$; it can be efficiently implemented via FFT.

Here, as for multiple kernels the matrix $\mb A_k$ is \emph{overcomplete} (i.e., $\mb A_k$ is a wide rectangular matrix), we \emph{cannot} normalize the channel-wise weight matrix $\mb A_k$ to exact orthogonal. However, it can be normalized to \emph{tight frame} with $\mb Q_k \mb Q_k^\top \;=\; \mb I $. This further implies that we can normalize the spectral norm $\norm{\mb Q_k}{}$ of the weight matrix $\mb Q_k$ in each channel to unity (see \Cref{fig:Lipschitz} (Left)).

\begin{figure}[t]
	\centering
	\centering
    \includegraphics[width = 0.35\textwidth, ]{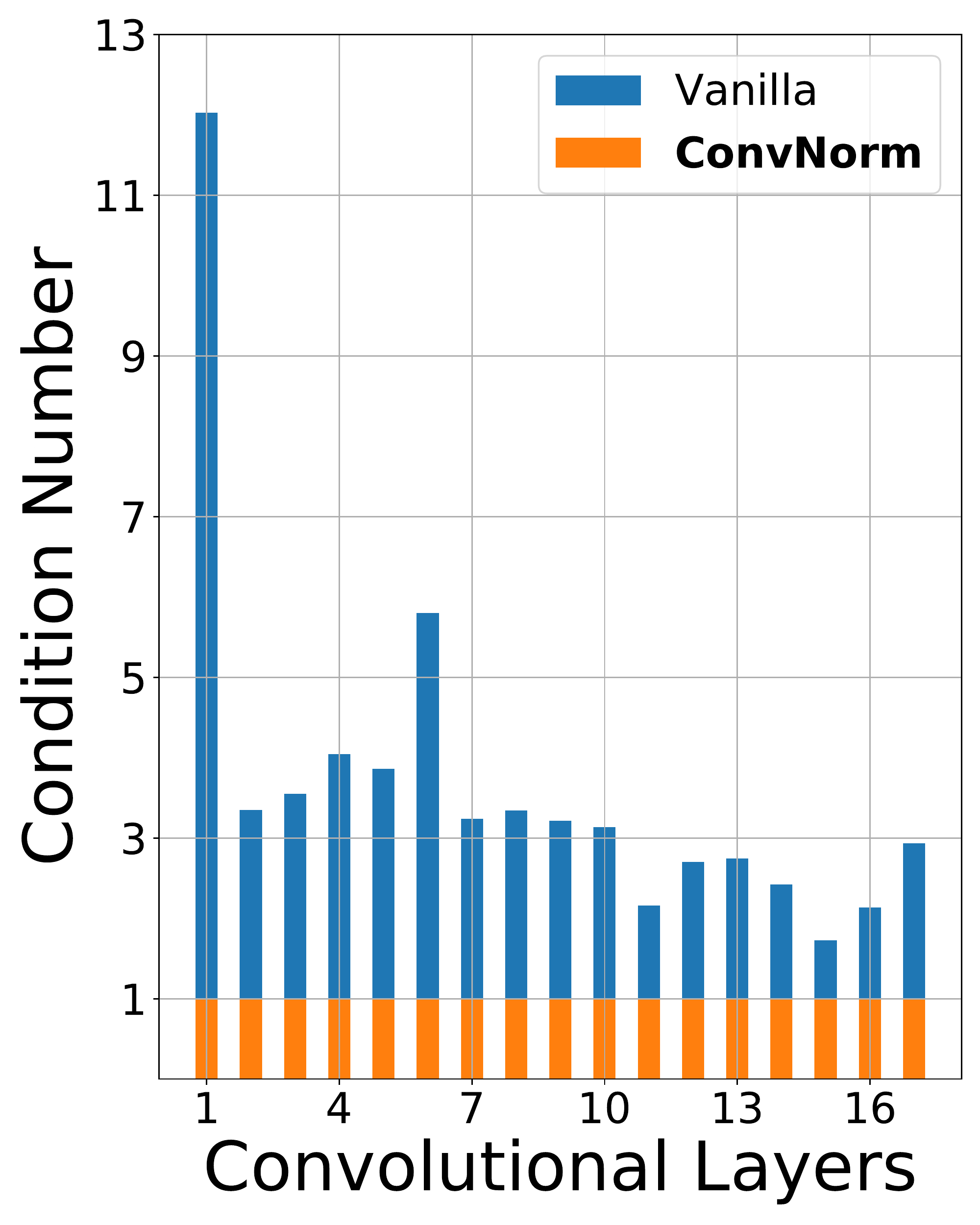}
    \includegraphics[width = 0.35\textwidth]{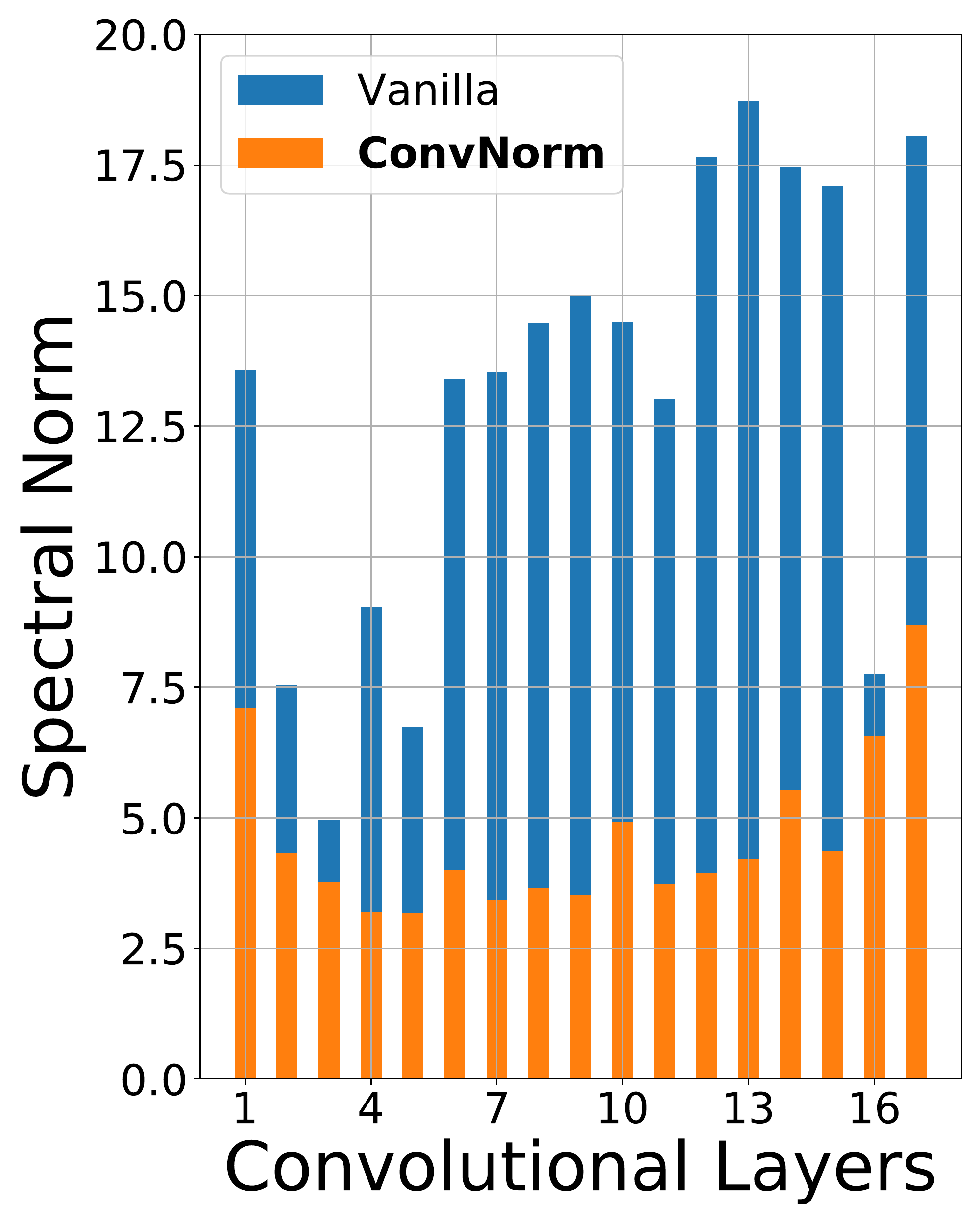}
	\caption{\textbf{Condition number for each channel (averaged over all channels) (Left), and spectral norm for each layer (Right)} on ResNet18 except for skip connection layers. ConvNorm normalizes the channel-wise condition number to 1 and reduces the layer-wise spectral norm. We use the method in \cite{sedghi2018singular} to calculate the singular values of the weight matrix.}
	\label{fig:Lipschitz}
\end{figure}

Combining the operation for all the channels, the ConvNorm for each layer overall can be summarized as follows:
\begin{align}\label{eqn:layer-wise-normalization}
    \wt{\mb z}_{out} \;=\; \begin{bmatrix}
        \mb P_1 \mb z_{1,out} \\
        \vdots \\
        \mb P_{C_O} \mb z_{C_O,out}
    \end{bmatrix}\;=\; \underbrace{ \begin{bmatrix}
        \mb Q_1 \\
        \vdots \\
        \mb Q_{C_O}
    \end{bmatrix} }_{\mb Q} \mb z_{in},
\end{align}

\vspace{-0.1in}
that we normalize each output channel $k$ by different matrix $\mb P_k$. The proposed ConvNorm has several advantages that we discuss below.

\vspace{-0.1in}
\begin{proposition}\label{prop:norm-bound}
The spectral norm of $\mb Q$ introduced in \eqref{eqn:layer-wise-normalization} can be bounded by
\vspace{-0.05in}
\begin{align*}
    \norm{\mb Q}{} \;\leq\; \sqrt{ \sum_{k=1}^{C_O} \norm{\mb Q_k}{}^2 },
\end{align*}
that spectral norm of $\mb Q$ is bounded by the spectral norms of all the weights $\Brac{\mb Q_k}_{k=1}^{C_O}$. 
\end{proposition}
\vspace{-0.1in}
\begin{proof}
We defer the proof to the Appendix \ref{subsec:proof_31}. 
\end{proof}

\vspace{-0.15in}
\begin{itemize}[leftmargin=*]
    \item \textbf{Efficient implementations.} There are many existing results \cite{huang2017orthogonal, huang2020controllable, trockman2021orthogonalizing} trying to normalize the whole layerwise weight matrix. For ConvNets, as the matrix is neither circulant nor block circulant, computing its inversion is often computationally prohibitive. Here, for each layer, we only normalize the weight matrix of the individual output channel. Thus similar to \Cref{subsec:warm-up}, the inversion in \eqref{eqn:Conv-Norm-k} can be much more efficiently computed via FFT by exploiting the benign convolutional structure.
    \vspace{-0.06in}
    \item \textbf{Improving layer-wise Lipschitzness.} As we can see from Proposition \ref{prop:norm-bound}, although ConvNorm only normalized the spectral norm of each channel, it can actually reduce the spectral norm of the whole weight matrix, improving the Lipschitzness of each layer; see \Cref{fig:Lipschitz} (Right) for a numerical demonstration on ResNet18. As extensively investigated \cite{cisse2017parseval, li2019preventing, trockman2021orthogonalizing}, improving the Lipschitzness of the weights for ConvNets will lead to enhanced robustness against data corruptions, for which we will demonstrate on the proposed ConvNorm in \Cref{subsec:exp-robustness}.
    \vspace{-0.06in}
    \item \textbf{Easier training and better generalization.} For deconvolution and convolutional sparse coding problems, the work \cite{qu2019nonconvex,qu2020geometric} showed that ConvNorm could dramatically improve the corresponding nonconvex optimization landscapes. On the other hand, from an algorithmic unrolling perspective for neural network design \cite{gregor2010learning, monga2019algorithm}, the ConvNorm is analogous to the preconditioned or conjugate gradient methods \cite{nocedal2006numerical} which often substantially boost algorithmic convergence. Therefore, we conjecture that the ConvNorm also leads to better optimization landscapes for training ConvNets, that they can be optimized faster to better solution qualities of generalization. We empirically show this in \Cref{subsec:training}.

\end{itemize}

\subsection{Extra technical details}

To achieve the full performance and efficiency potentials of the proposed ConvNorm, we discuss some essential implementation details in the following. 

\paragraph{Efficient back-propagation.} For ConvNorm, as the normalization matrix in \eqref{eqn:Conv-Norm-k} is constructed from the learned kernels, it complicates the computation of the gradient in back-propagation when training the network. Fortunately, we observe that treating the normalization matrices $\Brac{\mb P_k}$ as constants during back-propagation usually does \emph{not} affect the training and generalization performances, so that the computational complexity in training is not increased. We noticed that such a technique has also been recently considered in \cite{chen2020exploring} for self-supervised learning, which is termed as \emph{stop-gradient}.

\paragraph{Learnable affine tranform.} \label{subsec:extra}
For each channel, we include an (optional) \emph{affine transform} after the normalization $\mb P_k \cdot {\mb z}_{k,out} = \mb C_{\mb v_k} \cdot {\mb z}_{k,out} = \mb v_k \ast {\mb z}_{k,out} $ in \eqref{eqn:precond-multi} as follows:

\begin{align*}
   \overline{\mb z}_k\;=\;  \mb r_k\ast \wt{\mb z}_{k,out} \;=\; \mb r_k \ast \mb v_k \ast {\mb z}_{k,out},
\end{align*}
where the extra convolutional kernel $\mb r_k$ is learned along with the original model parameters. The idea of including this affine transform is analogous to including a learnable rescaling in BatchNorm, which can be considered as an "undo" operation to make sure the identity transform can be represented~\cite{ioffe2015batch}. The difference between our affine transform and BatchNorm is that we apply channel-wise convolutions instead of simple rescaling (see Figure~\ref{fig:ConvNorm}). Note that when $\mb r_k$ is an inverse kernel of $\mb v_k$ (i.e., ${\mb r}_k \ast \mb v_k = \mb 1 $), the overall transformation becomes an identity. The effectiveness of affine transform is demonstrated in the ablation study in Appendix \ref{subsec:app_class}. 

\paragraph{Dealing with stride and 2D convolution.} 
There are extra technicalities that we briefly discuss below. For more details, we refer the readers to Appendix \ref{app:implement}.
\vspace{-0.1in}
\begin{itemize}[leftmargin=*]
    \item \emph{Extension to 2D convolution.} Although we introduced the ConvNorm based on 1D convolution for the simplicity of presentations, it should be noted that our approach can be easily extended to the 2D case via 2D FFT.
    \vspace{-0.05in}
    \item \emph{Dealing with stride.} Strided convolutions are universal in modern ConvNet architectures such as the ResNet~\cite{he2016deep}, which can be viewed as downsampling after unstrided convolution. To deal with stride for our ConvNorm, we first perform an unstrided convolution, normalizing the activations using ConvNorm and then downsampling the normalized activations. In comparison, the method proposed in \cite{trockman2021orthogonalizing} is incompatible with strided convolutions.
\end{itemize}

\section{Experiments \& Results}\label{sec:exp}
In this section, we run extensive experiments on CIFAR and ImageNet, empirically demonstrating two major advantages of our approach: \emph{(i)} it improves the \emph{robustness} against adversarial attacks, data scarcity, and label noise corruptions \cite{guo2019simple,goodfellow2015explaining,madry2018towards}, and \emph{(ii)} it makes deep ConvNets \emph{easier to train} and perform better on problems such as classification and GANs~\cite{Goodfellow2014GenerativeAN}. The rest of this section is organized as follows. First, we introduce baseline methods for comparisons, and describe the setups of network architectures, datasets, and training. In \Cref{subsec:exp-robustness} and \Cref{subsec:training}, we demonstrate the effectiveness of our approach on robustness and training, respectively.


\paragraph{Baseline methods for comparisons.}
We compare our method with three representative methods.

\begin{itemize}[leftmargin=*]
    \item \textbf{Spectral normalization (SN).} For each layer of ConvNets, the work \cite{miyato2018spectral} treats multi-dimensional convolutional kernels as 2D matrices (e.g.,  by flattening certain dimensions) and normalizes its spectrum (i.e., singular values). It estimates the matrix's maximum singular value via a power method and then uses it to normalize all the singular values. As we discussed in \Cref{sec:intro}, the method does not exploit convolutional structures of ConvNets.
  
    \item \textbf{Orthogonalization by Newton's Iteration (ONI).} The work \cite{huang2020controllable} whitens the same reshaped matrices as SN, so that the reshaped matrices are reparametrized to orthogonality. However, the method needs to compute full inversions of covariance matrices, which is approximated by Newton's iterations. Again, no convolutional structure is utilized.

    \item \textbf{Orthogonal ConvNets (OCNN).} Few methods that exploit convolutional structures are \cite{wang2019orthogonal,qi2020deep}, which enforce orthogonality on doubly block circulant matrices of kernels via penalties on the loss. Here, we compare with \cite{wang2019orthogonal}.
\end{itemize}

\paragraph{Setups of dataset, network and training.} For all experiments, if not otherwise mentioned, CIFAR-10 and CIFAR-100 datasets are processed with standard augmentations, i.e., random cropping and flipping. We use 10\% of the training set for validation and treat the validation set as a held-out test set. For ImageNet, we perform standard random resizing and flipping. For training, we observe our ConvNorm is not sensitive to the learning rate, and thus we fix the initial learning rate to $0.1$ for all experiments.\footnote{For experiments with ONI, we use learning rate $0.01$ since the loss would be trained to NaN if with $0.1$.} For experiments on CIFAR-10, we run $120$ epochs and divide the learning rate by $10$ at the $40$th and $80$th epochs; for CIFAR-100, we run $150$ epochs and divide the learning rate by $10$ at the $60$th and $120$th epoch; for ImageNet,we run $90$ epochs and divide the learning rate by $10$ at the $30$th and $90$th epochs. The optimization is done using SGD with a momentum of $0.9$ and a weight decay of $0.0001$ for all datasets. For networks we use two backbone networks: VGG16~\cite{Simonyan2015VeryDC} and ResNet18~\cite{he2016deep}. We adopt Xavier uniform initialization~\cite{Glorot2010UnderstandingTD} which is the default initialization in PyTorch for all networks.


\begin{center}
\begin{table}[t]
\centering
\resizebox{0.80\linewidth}{!}{

 \begin{tabular}{c c c c c c c} 
\toprule
$\epsilon$ & Test Acc. & SN & BN & ONI & OCNN & ConvNorm\\ 
 \midrule
 \midrule
 0 & Clean & 82.52 $\pm$ 0.22 & 82.13 $\pm$ 0.67 & 80.70 $\pm$ 0.14 & 82.90 $\pm$ 0.31 & \textbf{83.23 $\pm$ 0.25}\\
 \midrule
 \multirow{4}{*}{{$\frac{8}{255}$}}
  & FGSM & 52.34 $\pm$ 0.33 & 51.72 $\pm$ 0.52 & 48.33 $\pm$ 0.16 & 52.49 $\pm$ 0.21 & \textbf{52.87 $\pm$ 0.24} \\
 \cmidrule{2-7}
 & PGD-10 & 45.68 $\pm$ 0.40 & 45.31 $\pm$ 0.29 & 42.30 $\pm$ 0.24 & 45.74 $\pm$ 0.13 & \textbf{46.12 $\pm$ 0.26} \\
  \cmidrule{2-7}
 & PGD-20 & 44.47 $\pm$ 0.37 & 44.04 $\pm$ 0.24 & 41.08 $\pm$ 0.30 & 44.53 $\pm$ 0.10 & \textbf{44.75 $\pm$ 0.30} \\
 \bottomrule 
\end{tabular}} 

\caption{\textbf{Comparison of ConvNorm to baseline methods under different gradient based attacks.} Models are robustly trained following the procedure in~\cite{shafahi2019adversarial} using a ResNet18 backbone. Experiments are conducted on CIFAR-10 dataset. Results are averaged over $4$ random seeds.}
\label{table:adv_training}

\end{table} 
\end{center}

\begin{table}[t]
\centering
\resizebox{0.50\linewidth}{!}{

 \begin{tabular}{c c c} 
\toprule
 Method & Average Queries & Attack Success rate (\%) \\ 
 \midrule
 \midrule
 
 SN & 2519.32 & 60.60 \\
 \midrule
 ONI & 2817.09 & 55.90 \\ 
 \midrule
 OCNN & 2892.81 & 54.50 \\
 \midrule
 ConvNorm & \textbf{2966.16} & \textbf{53.50} \\

 \bottomrule
\end{tabular}}

\caption{\textbf{Comparison of ConvNorm to baseline methods on SimBA black box attack.} The mean value of average queries (the higher, the better) and attack success rate (the lower, the better) throughout $3$ runs are reported. Models are trained using a ResNet18 backbone without BN layers.}
\label{table:simba}

\end{table} 
\subsection{Improved robustness}\label{subsec:exp-robustness}

In this section, we demonstrate our method is more robust to various kinds of adversarial attacks, as well as random label corruptions and small training datasets. 

\paragraph{Robustness against adversarial attacks.} Existing results \cite{trockman2021orthogonalizing, wang2019orthogonal} show that controlling the layer-wise Lipschitz constants for deep networks improves robustness against adversarial attack. Since our method improves the Lipschitzness of weights (see \Cref{fig:Lipschitz}), we demonstrate its robustness under adversarial attack on the CIFAR-10 dataset. We adopt both white-box (gradient based) attack~\cite{goodfellow2015explaining,madry2018towards} and black-box attack~\cite{guo2019simple} to test the robustness of our proposed method and other baseline methods. The results are presented in \Cref{table:adv_training} and \Cref{table:simba}. For the ease of presentation, all technical details about model training and generation of the adversarial examples are postponed to Appendix \ref{app:exp}.

In the case of gradient based attacks, we follow the training procedure described in~\cite{shafahi2019adversarial} to train models with our ConvNorm and other baseline methods. We report the performances of the robustly trained models on both the clean test dataset and datasets that are perturbed by Fast Gradient Sign Method (FGSM)~\cite{goodfellow2015explaining} and Projected Gradient Method (PGD)~\cite{madry2018towards}. As shown in \Cref{table:adv_training}, our ConvNorm outperforms other methods in terms of robustness under white-box attack while maintaining a good performance on clean test accuracy.

For black-box attack, we adopt a popular black-box adversarial attack method, Simple Black-box Adversarial Attacks (SimBA) \cite{guo2019simple}. By submitting queries to a model for updated test accuracy, the attack method iteratively finds a perturbation where the confidence score drops the most. We report the average queries and success rate after $3072$ iterations in \Cref{table:simba}. As we can see, the ConvNorm resists the most queries, and that the SimBA has the lowest attack success rate for ConvNorm compared with other baseline methods.


\begin{figure}[t]
\begin{tabular}{>{\centering\arraybackslash}m{0.45\linewidth} >{\centering\arraybackslash}m{0.45\linewidth}}
    \quad\quad {\small Noisy Label} & \quad\quad {\small Data Scarcity}  \\
    \includegraphics[width=0.8\linewidth]{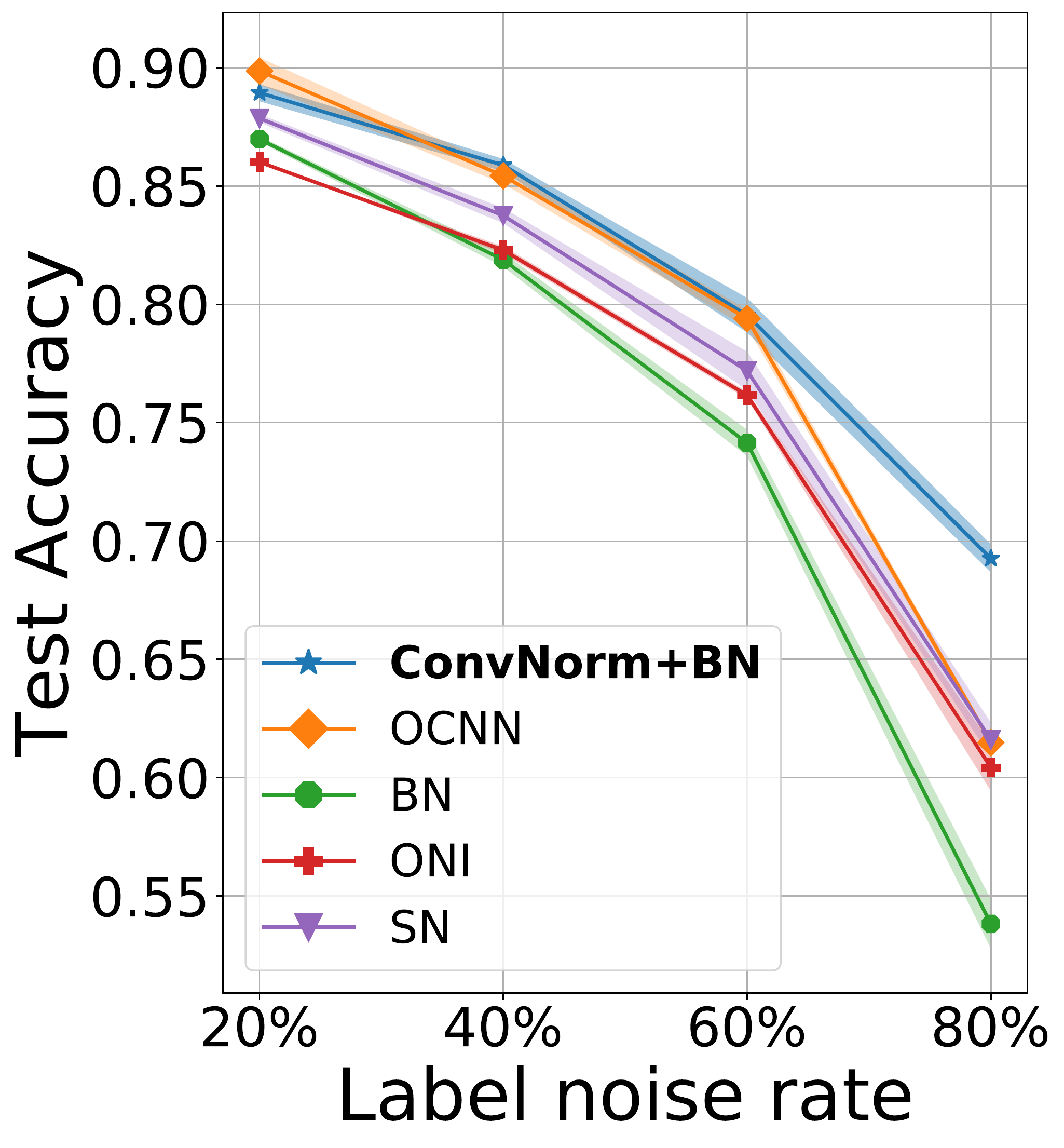} &
   \includegraphics[width=0.8\linewidth]{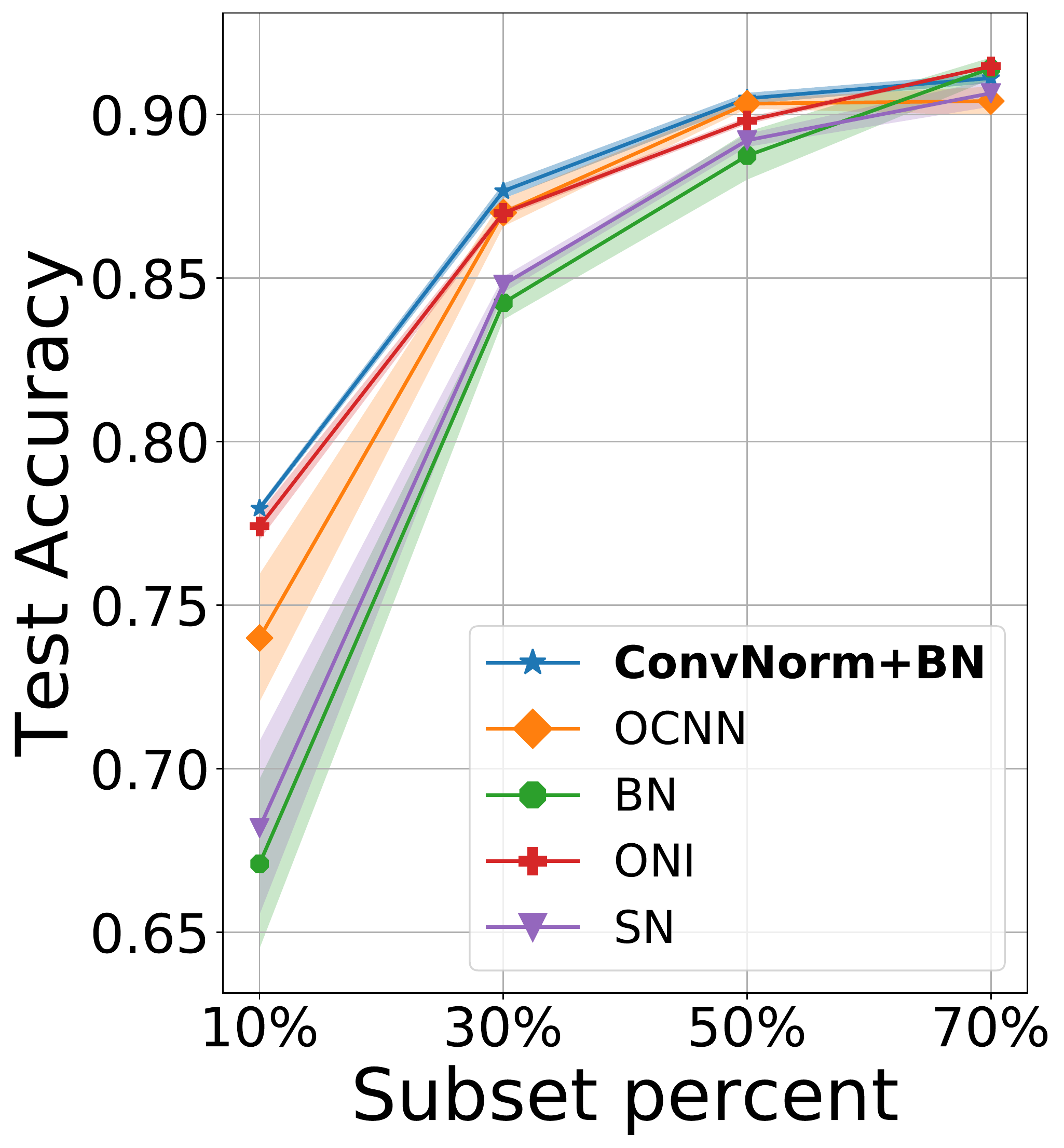} \\
    \end{tabular}

	\caption{\textbf{Test accuracy for noisy label (Left) and insufficient training data (Right)}. Experiments are conducted on CIFAR-10 dataset using a ResNet18 backbone. Error bars corresponding to standard deviations over 3 runs.}
	\label{fig:data_scar_noise_label}

\end{figure}

\paragraph{Robustness against label noise and data scarcity.}

It has been widely observed that overparameterized ConvNets tend to overfit when label noise presents or the amount of training labels is limited~\cite{szegedy2016rethinking, arpit2017closer,li2020gradient, liu2020early}. Recent work~\cite{hu2019simple} shows that normalizing the weights enforces certain regularizations, which can improve generalization performance against both label noise and data scarcity. Since our method is essentially reparametrizing and normalizing the weights, we demonstrate the robustness of our approach under these settings on CIFAR-10 with ResNet18 backbone.

\begin{itemize}[leftmargin=*]

    \item \emph{Robustness against label noise.} Following the scheme proposed in~\cite{patrini2017making}, we simulate noisy labels by randomly flipping 20\% to 80\% of the labels in the training set. As shown in \Cref{fig:data_scar_noise_label} (Left), our method outperforms the others on most noisy rates by a hefty margin when the noise level is high.

    \item \emph{Robustness against data scarcity.} We test our method on training the network with varying sizes of the training set, obtained by randomly sampling. The results in \Cref{fig:data_scar_noise_label} (Right) show that our ConvNorm achieves on par performance compared with baseline methods, and its performance stays high even when the size of the training data is tiny (e.g., $4500$ examples).
\end{itemize}


\subsection{Easier training on classification and GAN}


\label{subsec:training}
\begin{figure}[t]
\begin{tabular}{>{\centering\arraybackslash}m{0.31\linewidth} >{\centering\arraybackslash}m{0.31\linewidth}>{\centering\arraybackslash}m{0.31\linewidth}}
    \quad CIFAR-10 (VGG16) & \quad\quad CIFAR-10 (ResNet18) &  ImageNet (ResNet18) \\
    \includegraphics[width=0.95\linewidth]{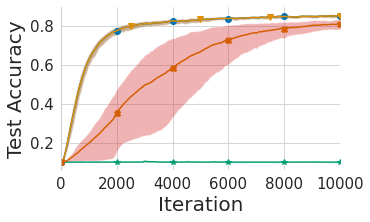} &
   \includegraphics[width=0.95\linewidth]{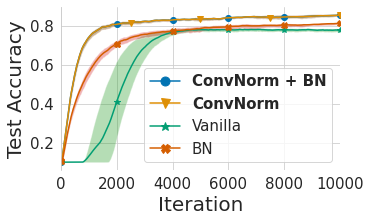} &
   \includegraphics[width=1\linewidth]{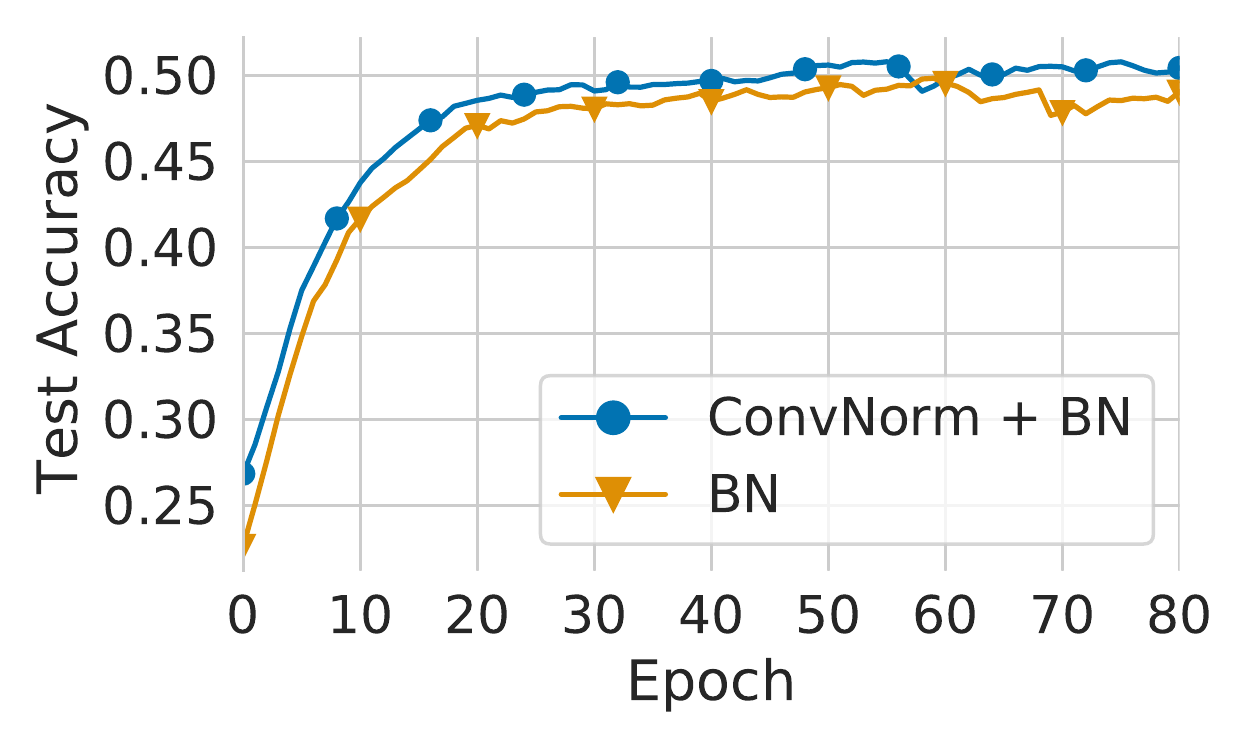} \\
    \end{tabular}

	\caption{\textbf{ConvNorm accelerates convergence.} VGG16 trained on CIFAR-10 (Left), and ResNet18 (Middle) trained on CIFAR-10 and ImageNet (Right),  with and without ConvNorm or BatchNorm. We do not use data augmentation, weight decay, or any other regularization in this experiment to isolate the effects of the normalization techniques. Error bars correspond to min/max over 4 runs.}
	\label{fig:compare-precond}

\end{figure}

Finally, we compare training convergence speed for classification and performances on GAN. Extra experiments on better generalization performance and ablation study can be found in Appendix \ref{subsec:app_class}.

\paragraph{Improved training on supervised learning.}
We test our method on image classification tasks with two backbone architectures: VGG16 and ResNet18.
We show that ConvNorm accelerates the convergence of training. To isolate the effects of the normalization layers for training, we train on CIFAR-10 and ImageNet without using any augmentation, regularization, and learning rate decay. In~\Cref{fig:compare-precond}, we show that adding ConvNorm consistently results in faster convergence, stable training (less variance in accuracy), and superior performance. On CIFAR-10, there is a wide performance gap after the first few iterations of training: 1000 iterations of training with ConvNorm lead to generalization performance comparable to 8000 iterations of training using BatchNorm. In the case of standard settings where data augmentation, regularization and learning rate decay are added, we notice that using ConvNorm and BatchNorm together also yield better test performances compared to only using BatchNorm (See Appendix \ref{subsec:app_class} for details). Besides the convergence speed of training, the exact training time for different methods is another important factor for measuring the efficiency of such methods. To this end, we empirically compare the training time for different methods and report the results in Appendix \ref{app:ablation} and \Cref{tab:run_time}.

\paragraph{Improved performance for GANs.}

It has been found that improving the Lipschitz condition of the discriminator of GAN stabilizes its training~\cite{Zhou2019LipschitzGA}. For instance, WGAN-GP \cite{gulrajani2017improved} demonstrates that adding a gradient penalty (1-GP) regularization to enforce the 1-Lipschitzness of the discriminator stabilizes GAN training and prevents mode collapse. Subsequent works \cite{kodali2017convergence,petzka2018on} using variants of the 1-GP regularization also show their improvement in GAN. Later on,  \cite{miyato2018spectral} further reveals the performance of GAN can be significantly improved if the spectral norm (Lipschitz condition) of the discriminator network is strictly enforced to 1. As shown in \Cref{fig:Lipschitz}, the proposed ConvNorm also controls the Lipschitz condition of ConvNets. Therefore, we expect our method to also ameliorates the performance of GAN.

\begin{table}[t]
\centering
\resizebox{0.5\linewidth}{!}{
	\begin{tabular}	{c c c c c c} 
		\toprule	 	
			Metric & SN  
			  & ONI & OCNN & Vanilla & {ConvNorm} \\
       \midrule	
      \multirow{1}{*} IS & \textbf{8.12} & {7.07} & {7.54} & {7.13} & {7.62}\\
      \midrule	
      \multirow{1}{*} FID  & \textbf{14.53}  & {29.49} & {22.15} & {29.47} & {21.37}\\
	\bottomrule
	\end{tabular}}
	
	
	
    \caption{\textbf{Comparison of ConvNorm to baseline methods on GAN training.} Inception score (IS) (the higher, the better) and FID score (the lower, the better) of ResNet with different normalizations. For each pair of model and method, we generate $50k$ images $10$ times and compute the mean of IS.}
    \label{tab:gan_results}
\end{table}	

To demonstrate the effectiveness of the ConvNorm on GAN, we compare it with other baseline methods introduced previously. In our experiments, we adopt the same settings and architecture suggested in~\cite{miyato2018spectral} without any modification, and we use the inception score (IS) \cite{salimans2016improved}, and FID \cite{heusel2017gans} score for quantitative evaluation. As shown in Table \ref{tab:gan_results}, our ConvNorm achieves the second-best performance to SN.\footnote{The performance of GANs is highly sensitive to the computational budget and the hyperparameters of the networks \cite{lucic2018gans}, and the hyperparameters of SN is fine-tuned for CIFAR-10 while we use the same hyperparameters as SN.}

\section{Discussions \& Conclusion}\label{sec:conclusion}


In this work, we introduced a new normalization approach for ConvNets, which explicitly exploits translation-invariance properties of convolutional operators, leading to efficient implementation and boosted performances in training, generalization, and robustness. Our work has opened several interesting directions to be further exploited for normalization design of ConvNets: \emph{(i)} although we provided some high-level intuitions why our ConvNorm works, theoretical justifications are needed; \emph{(ii)} as our ConvNorm only promotes channel-wise ``orthogonality'', it would be interesting to utilize similar ideas to efficiently normalize the layerwise weight matrices by exploiting convolutional structures. We leave these questions for future investigations.

\section*{Acknowledgement}

Part of this work was done when XL and QQ were at Center for Data Science, NYU. SL, XL, CFG, and QQ were partially supported by NSF grant DMS 2009752. SL was partially supported by NSF NRT-HDR Award 1922658. CY acknowledges support from Tsinghua-Berkeley Shenzhen Institute Research Fund. ZZ acknowledges support from NSF grant CCF 2008460. QQ also acknowledges support of Moore-Sloan fellowship, and startup fund at the University of Michigan.

\newpage 

{\small
\medskip
\bibliographystyle{ieeetr}
\bibliography{nips,nsf_proposal}

\begin{thebibliography}{10}

\bibitem{krizhevsky2012imagenet}
A.~Krizhevsky, I.~Sutskever, and G.~E. Hinton, ``Imagenet classification with
  deep convolutional neural networks,'' in {\em Advances in neural information
  processing systems}, pp.~1097--1105, 2012.

\bibitem{simonyan2014very}
K.~Simonyan and A.~Zisserman, ``Very deep convolutional networks for
  large-scale image recognition,'' {\em arXiv preprint arXiv:1409.1556}, 2014.

\bibitem{szegedy2015going}
C.~Szegedy, W.~Liu, Y.~Jia, P.~Sermanet, S.~Reed, D.~Anguelov, D.~Erhan,
  V.~Vanhoucke, and A.~Rabinovich, ``Going deeper with convolutions,'' in {\em
  Proceedings of the IEEE conference on computer vision and pattern
  recognition}, pp.~1--9, 2015.

\bibitem{ioffe2015batch}
S.~Ioffe and C.~Szegedy, ``Batch normalization: Accelerating deep network
  training by reducing internal covariate shift,'' in {\em ICML}, pp.~448--456,
  2015.

\bibitem{he2016deep}
K.~He, X.~Zhang, S.~Ren, and J.~Sun, ``Deep residual learning for image
  recognition,'' in {\em Proceedings of the IEEE conference on computer vision
  and pattern recognition}, pp.~770--778, 2016.

\bibitem{xie2017aggregated}
S.~Xie, R.~Girshick, P.~Doll{\'a}r, Z.~Tu, and K.~He, ``Aggregated residual
  transformations for deep neural networks,'' in {\em Proceedings of the IEEE
  conference on computer vision and pattern recognition}, pp.~1492--1500, 2017.

\bibitem{huang2017densely}
G.~Huang, Z.~Liu, L.~Van Der~Maaten, and K.~Q. Weinberger, ``Densely connected
  convolutional networks,'' in {\em Proceedings of the IEEE conference on
  computer vision and pattern recognition}, pp.~4700--4708, 2017.

\bibitem{huang2020normalization}
L.~Huang, J.~Qin, Y.~Zhou, F.~Zhu, L.~Liu, and L.~Shao, ``Normalization
  techniques in training dnns: Methodology, analysis and application,'' {\em
  arXiv preprint arXiv:2009.12836}, 2020.

\bibitem{ba2016layer}
J.~L. Ba, J.~R. Kiros, and G.~E. Hinton, ``Layer normalization,'' {\em arXiv
  preprint arXiv:1607.06450}, 2016.

\bibitem{ulyanov2016instance}
D.~Ulyanov, A.~Vedaldi, and V.~Lempitsky, ``Instance normalization: The missing
  ingredient for fast stylization,'' {\em arXiv preprint arXiv:1607.08022},
  2016.

\bibitem{wu2018group}
Y.~Wu and K.~He, ``Group normalization,'' in {\em Proceedings of the European
  conference on computer vision (ECCV)}, pp.~3--19, 2018.

\bibitem{he2017mask}
K.~He, G.~Gkioxari, P.~Doll{\'a}r, and R.~Girshick, ``Mask r-cnn,'' in {\em
  Proceedings of the IEEE international conference on computer vision},
  pp.~2961--2969, 2017.

\bibitem{sun2020test}
Y.~Sun, X.~Wang, Z.~Liu, J.~Miller, A.~Efros, and M.~Hardt, ``Test-time
  training with self-supervision for generalization under distribution
  shifts,'' in {\em International Conference on Machine Learning},
  pp.~9229--9248, PMLR, 2020.

\bibitem{salimans2016weight}
T.~Salimans and D.~P. Kingma, ``Weight normalization: A simple
  reparameterization to accelerate training of deep neural networks,'' {\em
  arXiv preprint arXiv:1602.07868}, 2016.

\bibitem{miyato2018spectral}
T.~Miyato, T.~Kataoka, M.~Koyama, and Y.~Yoshida, ``Spectral normalization for
  generative adversarial networks,'' {\em arXiv preprint arXiv:1802.05957},
  2018.

\bibitem{araujo2021lipschitz}
A.~Araujo, B.~Negrevergne, Y.~Chevaleyre, and J.~Atif, ``On lipschitz
  regularization of convolutional layers using toeplitz matrix theory,'' 2021.

\bibitem{qian2018l2}
H.~Qian and M.~N. Wegman, ``L2-nonexpansive neural networks,'' in {\em
  International Conference on Learning Representations}, 2018.

\bibitem{saxe2013exact}
A.~M. Saxe, J.~L. McClelland, and S.~Ganguli, ``Exact solutions to the
  nonlinear dynamics of learning in deep linear neural networks,'' {\em arXiv
  preprint arXiv:1312.6120}, 2013.

\bibitem{pennington2018emergence}
J.~Pennington, S.~Schoenholz, and S.~Ganguli, ``The emergence of spectral
  universality in deep networks,'' in {\em International Conference on
  Artificial Intelligence and Statistics}, pp.~1924--1932, 2018.

\bibitem{xiao2018dynamical}
L.~Xiao, Y.~Bahri, J.~Sohl-Dickstein, S.~Schoenholz, and J.~Pennington,
  ``Dynamical isometry and a mean field theory of cnns: How to train
  10,000-layer vanilla convolutional neural networks,'' in {\em International
  Conference on Machine Learning}, pp.~5393--5402, 2018.

\bibitem{hu2020provable}
W.~Hu, L.~Xiao, and J.~Pennington, ``Provable benefit of orthogonal
  initialization in optimizing deep linear networks,'' in {\em International
  Conference on Learning Representations}, 2020.

\bibitem{qi2020deep}
H.~Qi, C.~You, X.~Wang, Y.~Ma, and J.~Malik, ``Deep isometric learning for
  visual recognition,'' in {\em International Conference on Machine Learning},
  pp.~7824--7835, PMLR, 2020.

\bibitem{arjovsky2016unitary}
M.~Arjovsky, A.~Shah, and Y.~Bengio, ``Unitary evolution recurrent neural
  networks,'' in {\em International Conference on Machine Learning},
  pp.~1120--1128, 2016.

\bibitem{vorontsov2017orthogonality}
E.~Vorontsov, C.~Trabelsi, S.~Kadoury, and C.~Pal, ``On orthogonality and
  learning recurrent networks with long term dependencies,'' in {\em
  International Conference on Machine Learning}, pp.~3570--3578, 2017.

\bibitem{helfrich2018orthogonal}
K.~Helfrich, D.~Willmott, and Q.~Ye, ``Orthogonal recurrent neural networks
  with scaled cayley transform,'' in {\em International Conference on Machine
  Learning}, pp.~1969--1978, PMLR, 2018.

\bibitem{lezcano2019cheap}
M.~Lezcano-Casado and D.~Mart{\i}nez-Rubio, ``Cheap orthogonal constraints in
  neural networks: A simple parametrization of the orthogonal and unitary
  group,'' in {\em International Conference on Machine Learning},
  pp.~3794--3803, 2019.

\bibitem{jia2019orthogonal}
K.~Jia, S.~Li, Y.~Wen, T.~Liu, and D.~Tao, ``Orthogonal deep neural networks,''
  {\em IEEE transactions on pattern analysis and machine intelligence}, 2019.

\bibitem{cisse2017parseval}
M.~Cisse, P.~Bojanowski, E.~Grave, Y.~Dauphin, and N.~Usunier, ``Parseval
  networks: Improving robustness to adversarial examples,'' {\em arXiv preprint
  arXiv:1704.08847}, 2017.

\bibitem{trockman2021orthogonalizing}
A.~Trockman and J.~Z. Kolter, ``Orthogonalizing convolutional layers with the
  cayley transform,'' in {\em International Conference on Learning
  Representations}, 2021.

\bibitem{liu2020oogan}
B.~Liu, Y.~Zhu, Z.~Fu, G.~de~Melo, and A.~Elgammal, ``Oogan: Disentangling gan
  with one-hot sampling and orthogonal regularization.,'' in {\em AAAI},
  pp.~4836--4843, 2020.

\bibitem{ye2020network}
C.~Ye, M.~Evanusa, H.~He, A.~Mitrokhin, T.~Goldstein, J.~A. Yorke,
  C.~Fermuller, and Y.~Aloimonos, ``Network deconvolution,'' in {\em
  International Conference on Learning Representations}, 2020.

\bibitem{brock2018large}
A.~Brock, J.~Donahue, and K.~Simonyan, ``Large scale gan training for high
  fidelity natural image synthesis,'' in {\em International Conference on
  Learning Representations}, 2018.

\bibitem{odena2018generator}
A.~Odena, J.~Buckman, C.~Olsson, T.~B. Brown, C.~Olah, C.~Raffel, and
  I.~Goodfellow, ``Is generator conditioning causally related to gan
  performance?,'' {\em arXiv preprint arXiv:1802.08768}, 2018.

\bibitem{atzmon2020isometric}
M.~Atzmon, A.~Gropp, and Y.~Lipman, ``Isometric autoencoders,'' {\em arXiv
  preprint arXiv:2006.09289}, 2020.

\bibitem{harandi2016generalized}
M.~Harandi and B.~Fernando, ``Generalized backpropagation, \'{E}tude de cas:
  Orthogonality,'' {\em arXiv}, 2016.

\bibitem{bansal2018can}
N.~Bansal, X.~Chen, and Z.~Wang, ``Can we gain more from orthogonality
  regularizations in training deep cnns?,'' in {\em Proceedings of the 32nd
  International Conference on Neural Information Processing Systems},
  pp.~4266--4276, Curran Associates Inc., 2018.

\bibitem{zhang2019approximated}
G.~Zhang, K.~Niwa, and W.~B. Kleijn, ``Approximated orthonormal normalisation
  in training neural networks,'' 2019.

\bibitem{Li2020Efficient}
J.~Li, F.~Li, and S.~Todorovic, ``Efficient riemannian optimization on the
  stiefel manifold via the cayley transform,'' in {\em International Conference
  on Learning Representations}, 2020.

\bibitem{huang2020controllable}
L.~Huang, L.~Liu, F.~Zhu, D.~Wan, Z.~Yuan, B.~Li, and L.~Shao, ``Controllable
  orthogonalization in training dnns,'' 2020.

\bibitem{qu2020geometric}
Q.~Qu, Y.~Zhai, X.~Li, Y.~Zhang, and Z.~Zhu, ``Geometric analysis of nonconvex
  optimization landscapes for overcomplete learning,'' in {\em International
  Conference on Learning Representations}, 2020.

\bibitem{huang2017orthogonal}
L.~Huang, X.~Liu, B.~Lang, A.~W. Yu, and B.~Li, ``Orthogonal weight
  normalization: Solution to optimization over multiple dependent stiefel
  manifolds in deep neural networks,'' {\em CoRR}, vol.~abs/1709.06079, 2017.

\bibitem{li2019preventing}
Q.~Li, S.~Haque, C.~Anil, J.~Lucas, R.~Grosse, and J.-H. Jacobsen, ``Preventing
  gradient attenuation in lipschitz constrained convolutional networks,'' {\em
  Conference on Neural Information Processing Systems}, 2019.

\bibitem{wang2019orthogonal}
J.~Wang, Y.~Chen, R.~Chakraborty, and S.~X. Yu, ``Orthogonal convolutional
  neural networks,'' 2019.

\bibitem{krizhevsky2009learning}
A.~Krizhevsky {\em et~al.}, ``Learning multiple layers of features from tiny
  images,'' 2009.

\bibitem{russakovsky2015imagenet}
O.~Russakovsky, J.~Deng, H.~Su, J.~Krause, S.~Satheesh, S.~Ma, Z.~Huang,
  A.~Karpathy, A.~Khosla, M.~Bernstein, {\em et~al.}, ``Imagenet large scale
  visual recognition challenge,'' {\em International journal of computer
  vision}, vol.~115, no.~3, pp.~211--252, 2015.

\bibitem{qu2019nonconvex}
Q.~Qu, X.~Li, and Z.~Zhu, ``A nonconvex approach for exact and efficient
  multichannel sparse blind deconvolution,'' in {\em Advances in Neural
  Information Processing Systems}, pp.~4017--4028, 2019.

\bibitem{srivastava2014dropout}
N.~Srivastava, G.~Hinton, A.~Krizhevsky, I.~Sutskever, and R.~Salakhutdinov,
  ``Dropout: a simple way to prevent neural networks from overfitting,'' {\em
  The journal of machine learning research}, vol.~15, no.~1, pp.~1929--1958,
  2014.

\bibitem{janocha2017loss}
K.~Janocha and W.~M. Czarnecki, ``On loss functions for deep neural networks in
  classification,'' {\em arXiv preprint arXiv:1702.05659}, 2017.

\bibitem{lecun1998gradient}
Y.~LeCun, L.~Bottou, Y.~Bengio, and P.~Haffner, ``Gradient-based learning
  applied to document recognition,'' {\em Proceedings of the IEEE}, vol.~86,
  no.~11, pp.~2278--2324, 1998.

\bibitem{lecun2015deep}
Y.~LeCun, Y.~Bengio, and G.~Hinton, ``Deep learning,'' {\em nature}, vol.~521,
  no.~7553, pp.~436--444, 2015.

\bibitem{sedghi2018singular}
H.~Sedghi, V.~Gupta, and P.~M. Long, ``The singular values of convolutional
  layers,'' {\em arXiv preprint arXiv:1805.10408}, 2018.

\bibitem{gregor2010learning}
K.~Gregor and Y.~LeCun, ``Learning fast approximations of sparse coding.,'' in
  {\em ICML}, pp.~399--406, 2010.

\bibitem{monga2019algorithm}
V.~Monga, Y.~Li, and Y.~C. Eldar, ``Algorithm unrolling: Interpretable,
  efficient deep learning for signal and image processing,'' {\em arXiv
  preprint arXiv:1912.10557}, 2019.

\bibitem{nocedal2006numerical}
J.~Nocedal and S.~J. Wright, {\em Numerical Optimization}.
\newblock New York, NY, USA: Springer, second~ed., 2006.

\bibitem{chen2020exploring}
X.~Chen and K.~He, ``Exploring simple siamese representation learning,'' {\em
  arXiv preprint arXiv:2011.10566}, 2020.

\bibitem{guo2019simple}
C.~Guo, J.~Gardner, Y.~You, A.~G. Wilson, and K.~Weinberger, ``Simple black-box
  adversarial attacks,'' in {\em Proceedings of the 36th International
  Conference on Machine Learning}, pp.~2484--2493, 2019.

\bibitem{goodfellow2015explaining}
I.~Goodfellow, J.~Shlens, and C.~Szegedy, ``Explaining and harnessing
  adversarial examples,'' in {\em International Conference on Learning
  Representations}, 2015.

\bibitem{madry2018towards}
A.~Madry, A.~Makelov, L.~Schmidt, D.~Tsipras, and A.~Vladu, ``Towards deep
  learning models resistant to adversarial attacks,'' in {\em International
  Conference on Learning Representations}, 2018.

\bibitem{Goodfellow2014GenerativeAN}
I.~J. Goodfellow, J.~Pouget-Abadie, M.~Mirza, B.~Xu, D.~Warde-Farley, S.~Ozair,
  A.~C. Courville, and Y.~Bengio, ``Generative adversarial nets,'' in {\em
  NIPS}, 2014.

\bibitem{Simonyan2015VeryDC}
K.~Simonyan and A.~Zisserman, ``Very deep convolutional networks for
  large-scale image recognition,'' {\em CoRR}, vol.~abs/1409.1556, 2015.

\bibitem{Glorot2010UnderstandingTD}
X.~Glorot and Y.~Bengio, ``Understanding the difficulty of training deep
  feedforward neural networks,'' in {\em AISTATS}, 2010.

\bibitem{shafahi2019adversarial}
A.~Shafahi, M.~Najibi, M.~A. Ghiasi, Z.~Xu, J.~Dickerson, C.~Studer, L.~S.
  Davis, G.~Taylor, and T.~Goldstein, ``Adversarial training for free!,'' in
  {\em Advances in neural information processing systems}, vol.~32, 2019.

\bibitem{szegedy2016rethinking}
C.~Szegedy, V.~Vanhoucke, S.~Ioffe, J.~Shlens, and Z.~Wojna, ``Rethinking the
  inception architecture for computer vision,'' in {\em Proceedings of the IEEE
  conference on computer vision and pattern recognition}, pp.~2818--2826, 2016.

\bibitem{arpit2017closer}
D.~Arpit, S.~Jastrz{\k{e}}bski, N.~Ballas, D.~Krueger, E.~Bengio, M.~S. Kanwal,
  T.~Maharaj, A.~Fischer, A.~Courville, Y.~Bengio, {\em et~al.}, ``A closer
  look at memorization in deep networks,'' in {\em International Conference on
  Machine Learning}, pp.~233--242, PMLR, 2017.

\bibitem{li2020gradient}
M.~Li, M.~Soltanolkotabi, and S.~Oymak, ``Gradient descent with early stopping
  is provably robust to label noise for overparameterized neural networks,'' in
  {\em International Conference on Artificial Intelligence and Statistics},
  pp.~4313--4324, PMLR, 2020.

\bibitem{liu2020early}
S.~Liu, J.~Niles-Weed, N.~Razavian, and C.~Fernandez-Granda, ``Early-learning
  regularization prevents memorization of noisy labels,'' {\em Advances in
  Neural Information Processing Systems}, vol.~33, 2020.

\bibitem{hu2019simple}
W.~Hu, Z.~Li, and D.~Yu, ``Simple and effective regularization methods for
  training on noisily labeled data with generalization guarantee,'' {\em arXiv
  preprint arXiv:1905.11368}, 2019.

\bibitem{patrini2017making}
G.~Patrini, A.~Rozza, A.~Krishna~Menon, R.~Nock, and L.~Qu, ``Making deep
  neural networks robust to label noise: A loss correction approach,'' in {\em
  Proceedings of the IEEE Conference on Computer Vision and Pattern
  Recognition}, pp.~1944--1952, 2017.

\bibitem{Zhou2019LipschitzGA}
Z.~Zhou, J.~Liang, Y.~Song, L.~Yu, H.~Wang, W.~Zhang, Y.~Yu, and Z.~Zhang,
  ``Lipschitz generative adversarial nets,'' in {\em ICML}, 2019.

\bibitem{gulrajani2017improved}
I.~Gulrajani, F.~Ahmed, M.~Arjovsky, V.~Dumoulin, and A.~C. Courville,
  ``Improved training of wasserstein {GAN}s,'' in {\em Advances in neural
  information processing systems}, pp.~5767--5777, 2017.

\bibitem{kodali2017convergence}
N.~Kodali, J.~Abernethy, J.~Hays, and Z.~Kira, ``On convergence and stability
  of {GAN}s,'' {\em arXiv preprint arXiv:1705.07215}, 2017.

\bibitem{petzka2018on}
H.~Petzka, A.~Fischer, and D.~Lukovnikov, ``On the regularization of
  wasserstein {GAN}s,'' in {\em International Conference on Learning
  Representations}, 2018.

\bibitem{salimans2016improved}
T.~Salimans, I.~Goodfellow, W.~Zaremba, V.~Cheung, A.~Radford, and X.~Chen,
  ``Improved techniques for training gans,'' {\em arXiv preprint
  arXiv:1606.03498}, 2016.

\bibitem{heusel2017gans}
M.~Heusel, H.~Ramsauer, T.~Unterthiner, B.~Nessler, and S.~Hochreiter, ``{GAN}s
  trained by a two time-scale update rule converge to a local nash
  equilibrium,'' in {\em Advances in neural information processing systems},
  pp.~6626--6637, 2017.

\bibitem{lucic2018gans}
M.~Lucic, K.~Kurach, M.~Michalski, S.~Gelly, and O.~Bousquet, ``Are {GAN}s
  created equal? a large-scale study,'' in {\em Advances in neural information
  processing systems}, pp.~700--709, 2018.

\bibitem{kong2017stride}
C.~Kong and S.~Lucey, ``Take it in your stride: Do we need striding in cnns?,''
  {\em arXiv preprint arXiv:1712.02502}, 2017.

\end{thebibliography}
}

\newpage

\appendices


The whole appendix is organized as follows. 
\begin{itemize}[leftmargin=*]
    \item In  Appendix~\ref{app:basic}, we introduce the basic notations that are used throughout the paper and the appendix, and introduce the basic tools for analysis.
    \item In Appendix~\ref{app:implement}, we describe the implementation details of our ConvNorm, including details for dealing with 2D convolutions, strides, paddings, and the discuss about the differences between different types of convolutions.
    \item In Appendix~\ref{app:exp}, we describe the the experimental settings for \Cref{sec:exp} in detail. 
    \item Finally, in Appendix~\ref{app:ablation} we conduct a more comprehensive ablation study on the influence of different components of the proposed ConvNorm.
\end{itemize}

\section{Notations \& basic tools} \label{app:basic}


\subsection{Notations}

Throughout this paper, all vectors/matrices are written in bold font $\mb a$/$\mb A$; indexed values are written as $a_i, A_{ij}$. For a matrix $\mb A \in \bb C^{m \times n}$, we use $\mb A^\top$ and $\mb A^*$ to denote the transpose and conjugate transpose of $\mb A$, respectively. We let $[m] =\Brac{1,2,\cdots,m}$. Let $\mb F_n \in \bb C^{n \times n}$ denote a unnormalized $n\times n$ DFT matrix, with $\norm{\mb F_n}{} = \sqrt{n}$, and $\mb F_n^{-1} = n^{-1}\mb F_n^*$. In many cases, we just use $\mb F$ to denote the DFT matrix. For any vector $\mb v\in \bb C^n$, we use $\wh{\mb v} = \mb F \mb v$ to denote its Fourier transform, and $\ol{\mb v}$ denotes the conjugate of $\mb v$. We use $\ast$ to denote the \emph{circular} convolution with modulo-$n$: $\paren{\mb v \ast \mb u}_i = \sum_{j=0}^{m-1} v_j u_{i-j}$, and we use $\conv$ to denote the cross-correlation $\mb v \conv \mb u$ used in modern ConvNets.


\subsection{Circular convolution and circulant matrices.}
 For a vector $\mb v \in \bb R^n$, let $\mathrm{s}_\ell[\mb v]$ denote the cyclic shift of $\mb v$ with length $\ell$.  Thus, we can introduce the circulant matrix $\mb C_{\mb v}\in \bb R^{n \times n}$ generated through $\mb v \in \bb R^n$, that is,
\begin{align*}\label{eqn:circulant matrx constrcut}
   \mb C_{\mb v} = \begin{bmatrix}
   v_1 & v_n & \cdots & v_3 & v_2 \\
   v_2 & v_1 & v_n  & & v_3 \\
   \vdots & v_2 & v_1 & \ddots &\vdots \\
   v_{n-1} &  & \ddots  & \ddots  &v_n\\
   v_n &  v_{n-1} & \cdots  & v_2 &v_1
 \end{bmatrix} = \begin{bmatrix}
 	\mathrm{s}_0\brac{ \mb v } & \mathrm{s}_1 \brac{\mb v} & \cdots & \mathrm{s}_{n-1}\brac{\mb v}
 \end{bmatrix}.
\end{align*}
Now the circular convolution can also be written in a simpler matrix-vector product form. For instance, for any $\mb u,\;\mb v \in \bb R^n$, we have
\begin{align*}
   \mb u \ast \mb v = \mb C_{\mb u}  \cdot  \mb v = \mb C_{\mb v} \cdot \mb u.
\end{align*}

In addition, the cross-correlation between $\mb u$ and $\mb v$ can be also written in a similar form of convolution operator which reverses one vector before convolution with $\check{\mb u}\ast \mb v$, where $\check{\mb v}$ denote a \emph{cyclic reversal} of $\mb v \in \bb R^m$ (i.e., $\check{\mb v} = \brac{ v_1, v_m, v_{m-1},\cdots,v_2 }^\top$). 


\subsection{Proof of Proposition~\ref{prop:norm-bound}} \label{subsec:proof_31}
We restate Proposition \ref{prop:norm-bound} in \Cref{sec:iso-bn} in the following.
\begin{proposition}
The spectral norm of $\mb Q$ introduced in \eqref{eqn:layer-wise-normalization} can be bounded by
\vspace{-0.1in}
\begin{align*}
    \norm{\mb Q}{} \;\leq\; \sqrt{ \sum_{k=1}^{C_O} \norm{\mb Q_k}{}^2 },
\end{align*}
that spectral norm of $\mb Q$ is bounded by the spectral norms of all the weights $\Brac{\mb Q_k}_{k=1}^{C_O}$. 
\end{proposition}

\begin{proof}
Suppose we have a matrix of the form $ \mb Q \;=\; \begin{bmatrix}
    \mb Q_1 \\
    \vdots \\
    \mb Q_{C_O}
    \end{bmatrix}$, then by using the relationship between singular values and eigenvalues,
\begin{align*}
    \sigma_1^2(\mb Q) \;&=\; \lambda_1\paren{ \mb Q^\top \mb Q }  \;=\; \lambda_1 \paren{ \sum_{i=1}^{C_O} \mb Q_i^\top \mb Q_i } 
    \;\leq\; \sum_{i=1}^{C_O} \lambda_1\paren{\mb Q_i^\top \mb Q_i} = \sum_{i=1}^{C_O} \sigma_1^2(\mb Q_i).
\end{align*}
Thus, we have
\begin{align*}
    \sigma_1(\mb Q) \;\leq \; \sqrt{ \sum_{i=1}^{C_O} \sigma_1^2(\mb Q_i)  },
\end{align*}
as desired.
\end{proof}


\section{Implementation details for \Cref{sec:iso-bn}}\label{app:implement}
In the main paper, for the ease of presentation we only introduced and discussed the high-level idea of the proposed ConvNorm, with few technical details missing. Here, we discuss the implementation details of the ConvNorm for ConvNets in practice. More specifically, Appendix \ref{sec:algorithm} provides the pseudocode of ConvNorm with circular convolutions, which is easy for presentation and analysis. It should be noted that modern ConvNets often use cross-correlation instead of circular convolutions. Hence in  Appendix \ref{subsec:rela_conv} and  Appendix \ref{subsec:padded_conv}, we discuss in detail on how to deal with this difference in practice. Additionally, in Appendix \ref{app:stride} and Appendix \ref{app:2d-kernel}, we include other implementation details, such as dealing with strides, and extensions from 1D to 2D convolutions.


\subsection{Algorithms \label{sec:algorithm}}
First of all, in \Cref{alg:algo_precondition} we provide detailed pseudocode of implementing the proposed ConvNorm in ConvNets for 2D input data, where the convolution operations are based on circular convolutions. From our discussion in \Cref{sec:iso-bn} , we can see that all the operations can be efficiently implemented in the frequency domain via 2D FFTs.\footnote{During evaluation, we use the moving average of $\wh{\mathbf{v}}_k$ during training, the momentum of the moving average is obtained by a cosine rampdown function $0.5\left(1+\text{cos}\left(\frac{\text{min}(\text{iter},40000)}{40000}\pi\right)\right)$, where iter is the current iteration.}

It should be noted that modern ConvNets often use cross-correlation rather than the circular convolution. Nonetheless, we discuss the differences and similarities between the two in the following. Based on this, we show how to adapt \Cref{alg:algo_precondition} to modern ConvNets (see Appendix~\ref{subsec:rela_conv}).


\begin{algorithm}[t]
\caption{\label{alg:algo_precondition}\ \ 
Pseudocode of the proposed \textbf{ConvNorm} in each layer of ConvNets with 2D inputs. 
}
\begin{tabbing}
\Req $\mathbf{z}_{out} \in \mathbb{R}^{B\times C_O \times W \times H}$ \quad\== {\small convolution outputs with batchsize $B$, channels $C_O$, width $W$, and height $H$}\\
\Req $\mathbf{a} \in \mathbb{R}^{C_O\times C_I \times k_1\times k_2}$\>= {\small kernels for all output channels $C_O$, input channels $C_{I}$, and kernel size $k_1 \times k_2$}\\
\Req $\mathbf{r}\in \mathbb{R}^{C_O\times k_1\times k_2}$ \> = {\small $C_O$ trainable kernels for affine transform with the same size of $\mb a$}\\
{\bf for} $k$ in $[1,\dots, C_O]$  \> \cm {\small for each output channel\\ {\bf do}}\\
\X $\wh{\mathbf{z}}_{k,out}\gets$ FFT$(\mathbf{z}_{k,out})$ \> \cm {\small apply 2D Fast Fourier Transform (FFT) on convolution output}\\
\X  $\wh{\mathbf{a}}_k \gets \text{FFT}(\mathbf{a}_k)$\>\cm {\small apply 2D FFT on kernels} \\
\X  $\wh{\mathbf{a}}_k \gets$ \verb"stop_gradient"$\left(\wh{\mathbf{a}}_k\right)$  \> \cm {\small treating ${\wh{\mb a}_k }$ as constants during back-propagation}  \\
\X $\wh{\mathbf{v}}_k \gets \left(\sum_{i=1}^{C_I} |\widehat{\mathbf{a}}_{ki}|^{\odot 2}\right)^{\odot -1/2}$ \> \cm {\small this is the 2D FFT of $\mathbf{v}_k$ }\\
\X $\wt{\mathbf{z}}_{k,out} \gets \text{IFFT} \left(\widehat{\mathbf{z}}_{k,out} \odot \wh{\mathbf{v}}_k\right)$ \>\cm {\small circularly convolve $\mathbf{z}_{out,k}$ with $\mathbf{v}_k$}\\
\X $\bar{\mathbf{z}}_{k,out} \gets \mathbf{r}_k \ast \wt{\mathbf{z}}_{k,out} $\> \cm {\small learnable affine transformation with $\mb r_k$}\\
\textbf{endfor}\\
\textbf{return} $\bar{\mb z}_{k,out}$ \>\cm {\small normalized convolution output}
\end{tabbing}
\end{algorithm}


\subsection{Dealing with convolutions in ConvNets}
\label{subsec:rela_conv}
Throughout the main body of the work, our description and analysis of ConvNorm are based on circular convolutions for the simplicity of presentations. However, it should be noted that current ConvNets typically use \emph{cross correlation} in each convolutional layer, which can be viewed as a variant of the classical linear convolution with flipped kernels. Hence, to adapt our analysis from circular convolution to cross-correlation (i.e., the typical convolution used in ConvNets), we need to build some sense of ``equivalence''between them. Since linear convolution has a close connection with both, we use linear convolution as a bridge to introduce the relationship and thus find the ``equivalence'' between circular convolution and cross-correlation. Based on this, we show how to adapt from circular convolution in \Cref{alg:algo_precondition} to the convolution used in modern ConvNets by simple modifications.


\paragraph{Relationship among all convolutions.} In the following presentations, assume we have a signal $\mb x \in \mathbb{R}^{n}$ and a kernel vector $\mb a \in \mathbb{R}^m$ with $m \leq n$. We first discuss the connections between linear convolution and the other two types of convolutions, and then establish the equivalence between circular convolution and cross-correlation upon the observed connections. \Cref{fig:ConvRelation} demonstrates a simple example of the connections.

\begin{itemize}[leftmargin=*]
\item \textbf{Linear convolution \& circular convolution.}
The (finite, discrete) linear and circular convolution can both be written as 
\begin{align*}
\mb y\paren{k} = \sum_{j=0}^{n-1}\mb a\paren{k-j} \mb x\paren{j}.
\end{align*}
Despite the same written form, they differ in two ways: length and index. As illustrated in \Cref{fig:ConvRelation} (i), linear convolution doesn't have constraints on the input length and the result always has length $n+m-1$. In comparison, circular convolution requires both the kernel $\mb a$ and the signal $\mb x$ to share the same length. Therefore, as in \Cref{fig:ConvRelation} (iii) and (iv), we always reduce linear convolution to circular convolution by zero padding both the kernel and the signal to the same length $n+m-1$, as shown in the example in \Cref{fig:ConvRelation} (iv). It should be noted that the reason for such length difference is also rooted in their different indexing methods. In the case of linear convolution, when indices fall outside the defined regions, the associated entries are $0$, e.g., $\mb a\paren{-1} = 0$ and $\mb x\paren{3} = 0$ as shown \Cref{fig:ConvRelation} (i). On the other hand, circular convolution uses the periodic indexing method, i.e., $\mb a\paren{-j} = \mb a\paren{m-j}$. For example, in \Cref{fig:ConvRelation} (iii), $\mb a\paren{-2} = \mb a\paren{3-2} = 4$. 

\item \textbf{Linear convolution \& cross-correlation.} As shown in \Cref{fig:ConvRelation} (i) and (ii), both linear convolution and cross-correlation operations apply the so-called \emph{sliding window} of the kernel $\mb a$ to the signal $\mb x$, where the sliding window moves to the right one step at a time when the stride equals one. However, notice that cross-correlation uses a flipped kernel compared with linear convolution. Another difference is in the length of the output, where the output of a linear convolution is of length $n+m-1$, while the output of cross-correlation is of length $n-m+1$. This is due to the fact that the cross-correlation operation does not calculate outputs for out-of-region indices (see the difference between \Cref{fig:ConvRelation} (i) and (ii) for an example). To sum up, a cross-correlation is equivalent to a kernel-flipped and truncated linear convolution. Moreover, the amount of truncation is controlled by the amount of zero-padding on the signal $\mb x$ in cross-correlation. For example, in \Cref{fig:ConvRelation} (ii), there is no zero-padding and hence the result is equivalent to truncate the first and last elements from the result in \Cref{fig:ConvRelation} (i); but consider if we zero-pad the signal $\mb x$ by $1$ element on both sides in \Cref{fig:ConvRelation} (ii), then the result would be identical with \Cref{fig:ConvRelation} (i). In general, we found that if we zero-pad the signal $\mb x$ by $m-1$ elements on both sides, a cross-correlation is equivalent to a kernel-flipped linear convolution without any truncation. We will discuss more about dealing with zero-padding in Appendix \ref{subsec:padded_conv}. 



\end{itemize}

\paragraph{Adapting circular convolution to cross-correlation in ConvNets.} Thus, based on these connections discussed above, we could now establish the ``equivalence'' between circular convolution and cross-correlation based on their connections to linear convolution, and hence adapt the proposed ConvNorm in \Cref{alg:algo_precondition} with cross-correlations in ConvNets via the following steps:

\begin{enumerate}
    \item Zero pad both sides of $\mb z_{in}$ by $m-1$ elements to get $\dot{\mb z}_{in}$.
    \item Perform the cross-correlation between the kernel $\mb a$ and the input $\dot{\mb z}_{in}$ to obtain the output $\mb z_{out}$.
    \item Apply ConvNorm on the kernel $\mb a$ and the output $\mb z_{out}$ as stated in \Cref{alg:algo_precondition} to get result $\wt{\mb z}_{out}$. 
    \item Delete the first and last $(m-1)/2$ elements of $\wt{\mb z}_{out}$ and return as the resulting output.
\end{enumerate}

Here, Step $1$ and Step $2$ are to generate the output $\mb z_{out}$ that is almost identical to what used in ConvNets, with the exception of zero-padding in Step $1$ so that it is equivalent to a circular convolution with a flipped kernel $\check{\mb a}$. Then in Step $3$, we perform ConvNorm on the output $\mb z_{out}$ and kernel $\mb a$. Notice that there is no need to flip the kernel in the above steps since as described in \Cref{alg:algo_precondition}, we only need to calculate the magnitude spectrum of a kernel and the magnitude spectrum remains consistent with a kernel flipping, i.e., $\abs{\mb F \mb a} = \abs{\mb F \check{\mb a}}$. Finally, Step $4$ is to obtain the desired output with the correct spatial dimension.

\begin{figure*}[t]
	\centering
    \includegraphics[width = \textwidth]{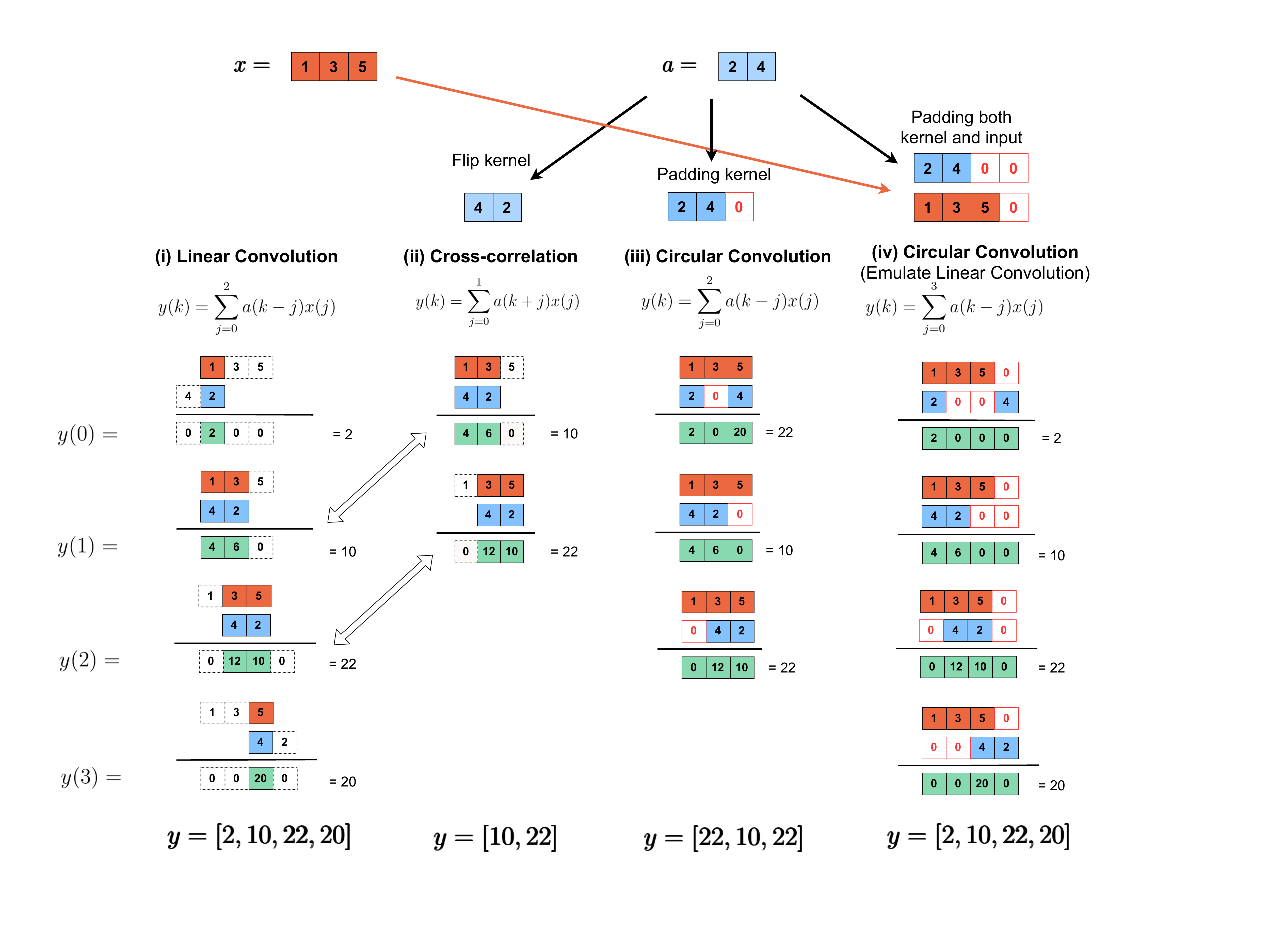}
    \vspace{-0.5in}
	\caption{\textbf{A 1D example to illustrate relationships among different kinds of convolutions.} (i), (ii), and (iii) show the operations of linear convolution, cross-correlation (i.e., the ``convolution'' used in ConvNets), and circular convolution, respectively. (iv) gives an example of emulating linear convolution in a circular convolution manner. (i), (ii) indicates that cross-correlation is essentially a linear convolution with a flipped-kernel and truncation. Hence from (i), (ii), (iv), we find equivalence between cross-correlation and circular convolution.}
	\label{fig:ConvRelation}
\end{figure*}

\subsection{Dealing with zero-paddings in ConvNets} 
\label{subsec:padded_conv}
Here, we provide more explanations about the zero-padding and truncation used in Appendix \ref{subsec:rela_conv}. Zero-padding is an operation of adding $0$s to the data, which is widely used in modern ConvNets primarily aimed for maintaining the spatial dimension of the outputs for each layer. For example, a standard stride-$1$ convolution in ConvNets between a kernel $\mb a \in \mathbb{R}^m$ and a signal $\mb z_{in} \in \mathbb{R}^{n}$ with $n > m$ produces a output vector $\mb z_{out}$ of length $n-m+1$. To make the output $\mb z_{out}$ the same length as the input signal $\mb z_{in}$, a zero-padding of $\lfloor m/2 \rfloor$ 
is often used (e.g., common in various architectures such as VGG\cite{Simonyan2015VeryDC}, ResNet\cite{he2016deep}.\footnote{$\lfloor m \rfloor$ is the floor operation which outputs the greatest integer less than or equal to $m$.} To handle with such zero-padding in ConvNorm, based on the relationship between cross-correlation and circular convolution established in Appendix \ref{subsec:rela_conv}, we truncate the output of ConvNorm to make its spatial dimension align with the dimension of the input signal $\mb z_{in}$. More specifically, after Step $3$ in Appendix \ref{subsec:rela_conv}, the resulting $\wt{\mb z}_{out}$ has length $n+m-1$, in Step $4$ we then truncate the first and last $(m-1)/2$ elements from it to make it has length $n$ as the input signal. 




\subsection{Dealing with stride-$2$}\label{app:stride}
Stride nowadays becomes an essential component in modern ConvNets \cite{he2016deep, kong2017stride}. A stride-$s$ convolution is a convolution with the kernel moving $s$ steps at a time instead of $1$ step in a standard convolution shown in \Cref{fig:ConvRelation}. Mathematically, for the kernel $\mb a$ and the input $\mb z_{in}$, the stride-$s$ convolution can be written as
\begin{align*}
    \mb z_{out} = \mc D_s\brac{ \mb a \conv \mb z_{in} },
\end{align*}
where $\mc D_s[\cdot]$ is a downsampling operator that selects every $s$th sample and discards the rest. Therefore, the main purpose of stride is for downsampling the output in ConvNets, replacing classical pooling methods. Hence for convolution with stride-$s$, the dimension of the output decreases by $s$ times in comparison to that of the standard stride-$1$ output. For example, if we do a stride-2 convolution on \Cref{fig:ConvRelation} (ii), we will get the result $\mb y=[10]$ where the result is sampled from the standard stride-$1$ convolution output and its size is halved.


When enforcing weight regularizations, recent work often cannot handle strided convolution \cite{trockman2021orthogonalizing}. This happens because it involves weight matrix inversion, and the stride and the downsampling operator cause the weight matrix to be non-invertible. In contrast, since our method does not involve computing full matrix inversion and it operates on the outputs instead of directly changing the convolutional weights, we could first take a step back to perform an unstrided convolution, then use ConvNorm to normalize the output and finally do the stride (downsampling) operation on the normalized outputs. 

\subsection{Dealing with 2D kernels}\label{app:2d-kernel}
Although in the main body of the work, we introduced the ConvNorm based on 1D convolution for the simplicity of presentations, it should be noted that our approach  can be easily extended to the 2D case via 2D FFT. For an illustration, let us consider \eqref{eqn:Conv-Norm-k}, we know that in 1D case the preconditioning matrix for each channel can be written in the form of
\begin{align*}
    \mb P_k \;&=\; \paren{\sum_{j=1}^{C_I} \mb C_{\mb a_{kj}} \mb C_{\mb a_{kj}}^\top}^{-\frac{1}{2}} \\
    \;&=\; \paren{\sum_{j=1}^{C_I} \mb F^{*}\diag(\wh{\mb a}_{kj})\mb F \mb F^{*} \diag(\overline{\wh{\mb a}}_{kj})\mb F}^{-\frac{1}{2}} =\; \mb F^{*}\paren{ \paren{\sum_{j=1}^{C_I} \abs{\diag(\wh{\mb a}_{kj})}^{\odot 2}}^{\odot -\frac{1}{2}}}\mb F,
\end{align*}
so that the output after ConvNorm can be rewritten as,
\begin{align*}
    \mb P_k \mb z_k \;&=\; \mb F^{*}\paren{ \sum_{j=1}^{C_I} \abs{\diag(\wh{\mb a}_{kj})}^{\odot 2})^{\odot -\frac{1}{2}} } \mb F \mb z_k \;= \mb F^{-1}\paren{\sum_{j=1}^{C_I} \abs{\mb F(\mb a_{kj})}^{\odot 2}}^{\odot -\frac{1}{2}} \mb F(\mb z_k),
\end{align*}
where $\mb F(\cdot)$ and $\mb F^{-1}(\cdot)$ denote the 1D Fourier transform and the 1D inverse Fourier transform, respectively. To extend our method to the 2D case, we can simply replace the 1D Fourier transform in the above equation by the 2D Fourier transform. As summarized in \Cref{alg:algo_precondition}, to deal with 2D input data, we replace every 1D Fourier transform with 2D Fourier transform, which can be efficiently implemented via 2D FFT.



\section{Experimental details for \Cref{sec:exp}}\label{app:exp}

In this part of appendix, we provide detailed descriptions for the choices of hyperparameters of baseline models, and introduce the settings for all experiments conducted in \Cref{sec:exp}.

\subsection{Computing resources, assets license}
We use two datasets for the demonstration purpose of this paper: CIFAR dataset is made available under the terms of the MIT license and ImageNet dataset is publicly available for free to researchers for non-commercial use. We refer the code of some work during various stages of our implementation for comparison and training purposes, we list them as follows: the implementation of ONI~\cite{huang2020controllable} is made available under the BSD-2-Clause license; the implementation of OCNN~\cite{wang2019orthogonal} is made available under the MIT license; the training procedure for \Cref{table:adv_training} refers to the implementation of the work~\cite{shafahi2019adversarial} which is made available under the MIT license and the black-box attack SimBA~\cite{guo2019simple} implementation is made available under the MIT license. All experiments are conducted using RTX-8000 GPUs. 

\subsection{Choice of hyperparameters for baseline methods}
In  \Cref{sec:exp}, we compare our method with three representative normalization methods, that we describe the hyperparameter settings of each method below.
\begin{itemize}[leftmargin=*]
    \item \textbf{OCNN.} Since the best penalty constraint constant $\lambda$ for OCNN is not specified in \cite{wang2019orthogonal}, we do a hyperparameter tuning on the clean CIFAR-10 dataset for $\lambda \in \{0.001, 0.01, 0.05, 0.1, 1\}$ and picked $\lambda = 0.01$ from the best validation set accuracy.
    \item \textbf{ONI.} In the work \cite{huang2020controllable}, the authors utilize Newton's iteration to approximate the inverse of the covariance matrix for the reshaped weights. In our experiments, we adopt the implementation and use the default setting from their \href{https://github.com/huangleiBuaa/ONI/blob/master/ONI_PyTorch/extension/normalization/NormedConv.py}{Github page} where the maximum iteration number of Newton's method is set to $5$. We use a learning rate $0.01$ for all ONI experiments, where we notice that a large learning rate $0.1$ makes the training loss explode to NaN.
    \item \textbf{SN.} In \cite{miyato2018spectral}, the authors use the power method to estimate the spectral norm of the reshaped weight matrix and then utilize the spectral norm to rescale the weight tensors. For all SN experiments, we directly use the official \href{https://pytorch.org/docs/stable/generated/torch.nn.utils.spectral_norm.html}{PyTorch implementation} of SN with the default settings where the iteration number is set to $1$.\footnote{The authors of SN take advantage of the fact that the change of weights from each gradient update step is small in the SGD case (and thus the change of the singular vector is small as well) and hence design the SN algorithm so that the approximated singular vector from the previous step is reused as the initial vector in the current step. They notice that $1$ iteration is sufficient in the long run.}
\end{itemize}



\subsection{Experimental details for \Cref{subsec:exp-robustness}}




\paragraph{Robustness against adversarial attacks.} For gradient-based attacks, we follow the training procedure described in~\cite{shafahi2019adversarial} to train models with our ConvNorm and other baseline methods.\footnote{For OCNN in the gradient-based attack experiment, we choose $\lambda = 0.0001$ since we found that this setting yields the best OCNN robust performance.} Then we use Fast Gradient Sign Method (FGSM)~\cite{goodfellow2015explaining} and Projected Gradient Method (PGD)~\cite{madry2018towards} as metrics to measure the robust performance of the trained models. We note that the FGSM attack is defined to find adversarial examples in one iteration by:
\begin{align*}
    \mb x_{adv} = \mb x + \epsilon * \sign(\nabla_{\mb x} \ell(\mb x, \mb y, \mb \theta))
\end{align*}
where $\mb \theta$ represents the model; $\mb y$ is the target for data $\mb x$ and $\epsilon$ denotes the attack amount of this iteration. PGD attack is an iterative version of FGSM with random noise perturbation as attack initialization. In this paper, we use PGD-$k$ to denote the total iterative steps (i.e., $k$ steps) for the attack methods. Adversarial attacks are always governed by a bound on the norm of the maximum possible perturbation, i.e., $\norm{\mb x_{adv} - \mb x}{p} \leq \delta_p$. We use $\ell_{\infty}$ norm to constrain the attacks throughout this paper (i.e., $p = \infty $). Specifically, we adopt the procedure in~\cite{shafahi2019adversarial} by choosing $m=4$ (the times of repeating training for each minibatch) and FGSM step $\epsilon = \frac{8}{255}$ during training. And we set the attack bound $\delta_{\infty} = \frac{8}{255}$.

For the black-box attack SimBA~\cite{guo2019simple}, we first train ResNet18 \cite{he2016deep} models for ConvNorm and other baseline methods without BatchNorm on the clean CIFAR-10 training images using the default experimental setting mentioned in \Cref{sec:exp}.\footnote{We choose to not adding BatchNorm in the SimBA experiment because we empirically observe that removing BatchNorm improves the performance of every method.} Then we choose the best model for each method according to the best validation set accuracy. Finally, we apply each of the selected models on the held-out test set and randomly pick $1000$ correctly classified test samples for running SimBA attack. We compare the performances of the selected models with pixel attack using a step size $\eps = 0.4$. Since images in the dataset have spatial resolution $32 \times 32$ and $3$ color channels, the attack runs in a total of $3 \times 32 \times 32 = 3072$ iterations. We then report the average queries and attack success rate after all $3072$ iterations in \Cref{table:simba}.

\paragraph{Robustness against label noise.} The label noise for CIFAR-10 is generated by randomly flipping the original labels. Here we show the specific definition. We inject the symmetric label noise to training and validation split of CIFAR-10 to simulate noisily labeled dataset. The symmetric label noise is as follows:
\begin{equation*}
 \mathbf{y} = \left \{
\begin{array}{lr}
     \mathbf{y}^{GT} \text{ with the probability of $1-r$},\\
    \text{random one-hot vector with the probability of $r$},\\
\end{array}
\right.
\end{equation*}
where $r \in [0,1]$ is the noise level. 
The models (using ResNet18 as backbones) are trained on noisily labeled training set (45000 examples) under the default experimental setting mentioned in \Cref{sec:exp}. Max test accuracy is then reported on the held-out test set. 
\paragraph{Robustness against data scarcity} We randomly sample [10\%, 30\%, 50\%, 70\%] of the training data set CIFAR-10 dataset while keeping the amount of validation and test set amount unchanged. The model is trained on sub-sampled training set using the default experimental setting mentioned in \Cref{sec:exp}. We report the accuracy on the held out test set by evaluating the best model selected on the validation set (obtained by randomly sampling 10\% of the original training set).

\begin{figure}[t]
\begin{tabular}{>{\centering\arraybackslash}m{0.45\linewidth} >{\centering\arraybackslash}m{0.45\linewidth}}
    \quad\quad Train & \quad\quad Test  \\
    \includegraphics[width=1.1\linewidth]{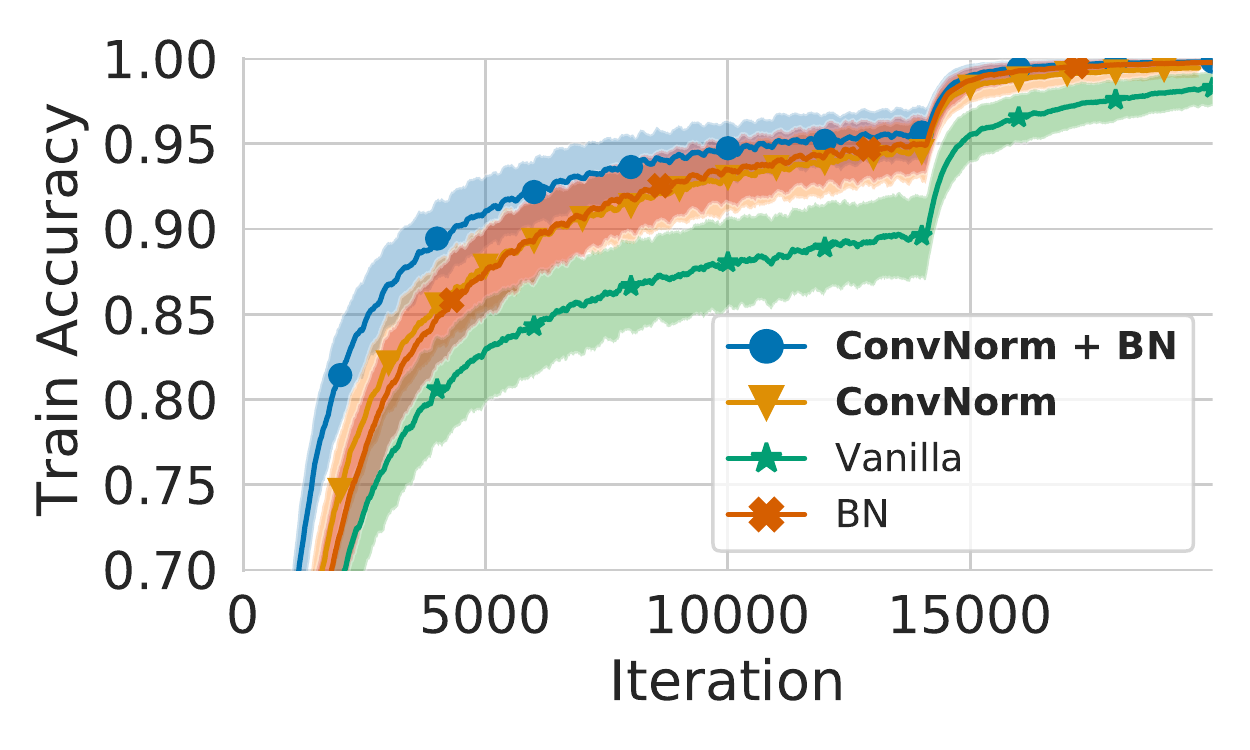} &
   \includegraphics[width=1.1\linewidth]{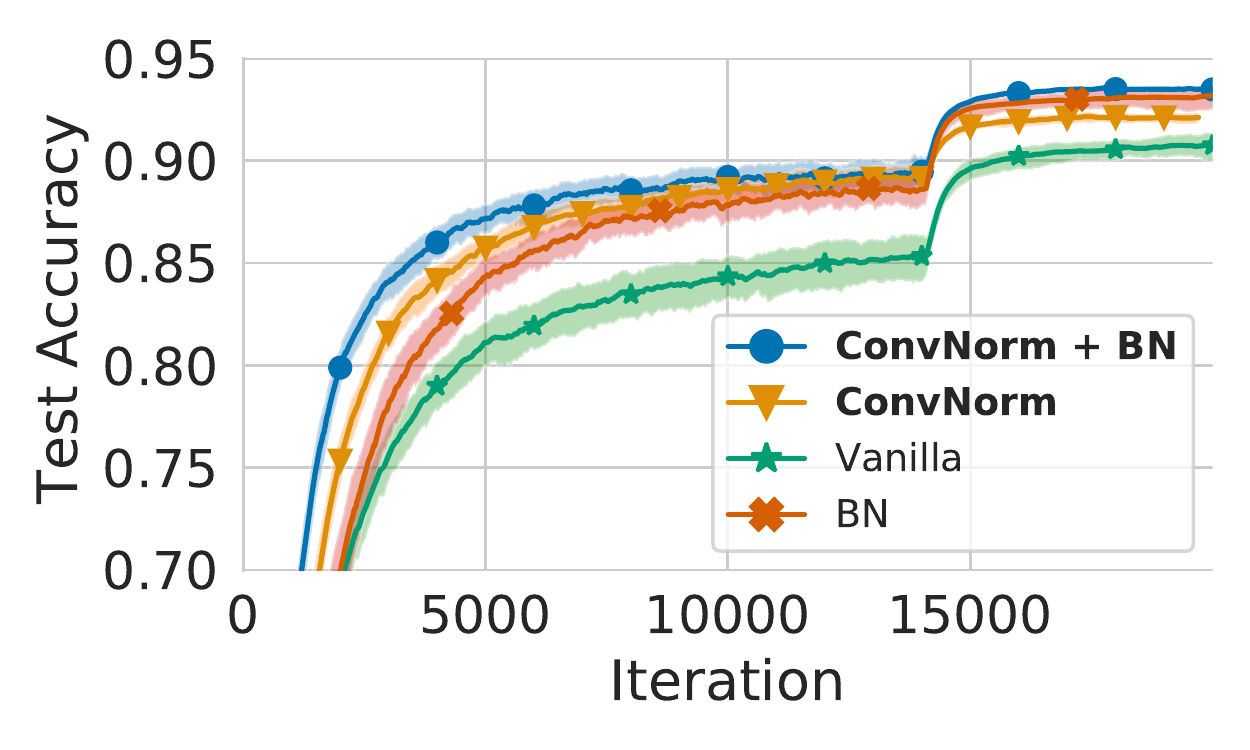} \\
    \end{tabular}
	\caption{\textbf{Adding ConvNorm before BatchNorm accelerates convergence and improves performance.} Train and test accuracy of ResNet18 trained on CIFAR-10 with and without ConvNorm or BatchNorm under default settings mentioned in~\Cref{sec:exp}. Error bars correspond to min/max over 3 runs.}
	\label{fig:ablation_bn}
	\vspace{-0.1in}
\end{figure}

\subsection{Generalization experiment and experimental details for \Cref{subsec:training}} \label{subsec:app_class}
\paragraph{Improved performances on supervised learning.}
Below we provide the generalization experiment and more detailed experiment settings for fast training and generalization mentioned in \Cref{subsec:training}. We use the default training setting mentioned in the \textbf{Setups of Dataset and Training} of \Cref{sec:exp} if not otherwise specified. 
\begin{itemize}[leftmargin=*]
    \item \emph{Faster training.} In order to isolate the effects of normalization techniques, we drop all regularization techniques including: data augmentations, weight decay, and learning rate decay as we have mentioned in the caption of ~\Cref{fig:compare-precond}. For extra experiments on the same analysis when these regularization techniques are included, please refer to Appendix~\ref{app:ablation} and the results in~\Cref{fig:ablation_bn}. 
    \item \emph{Better generalization.}
    We use the default setting as we have mentioned in the \textbf{Setups of Dataset and Training} of Section~\ref{sec:exp}. All normalization methods are evaluated under this setting. We demonstrate the test accuracy of our method on CIFAR and ImageNet under standard settings. As shown in \Cref{tab:generalization_results}, although only using ConvNorm results in slightly worse test accuracy against BatchNorm, adding ConvNorm before standard BatchNorm layers can boost the performance while maintaining fast convergence (see \Cref{fig:ablation_bn}). Additionally, we investigate the influence of combining the affine transform and BatchNorm with ConvNorm. Table~\ref{tab:ablation} shows the results of our ablation study on the CIFAR-10 dataset. Both affine transform and batch norm provide an independent performance boost.
\end{itemize}


\begin{table}[t]
\centering
\resizebox{0.96\linewidth}{!}{
	\begin{tabular}	{c c c c c c}
		\toprule	 	
			Dataset& Backbone& 
			 vanilla & BatchNorm(BN) & ConvNorm & ConvNorm + BN\\
			  \midrule			
	 CIFAR-10 & { ResNet18} &91.58 $\pm$ 0.67 & {93.18 $\pm$ 0.16} & 92.12 $\pm$ 0.32 & \textbf{93.31 $\pm$ 0.17}\\
       \midrule	
      CIFAR-100 & { ResNet18}& 66.59 $\pm$ 0.72 &{73.06 $\pm$ 0.13}  & 68.20 $\pm$ 0.27 & \textbf{73.38 $\pm$ 0.24}\\
		 \midrule	
     \multirow{1}{*}{ ImageNet} & { ResNet18}& / & {69.76}  & - & \textbf{70.34}\\

		\bottomrule
	\end{tabular}}
	\vspace{0.05in}
    \caption{\textbf{Results on classification.} Test accuracy on CIFAR-10, CIFAR-100, and ImageNet validation sets. For each case, we compare different combinations of ConvNorm and BatchNorm. Results of CIFAR-10 and CIFAR-100 are averaged over 4 random seeds, and "/" represents failed training.}
    \label{tab:generalization_results}
\end{table}	


\begin{table}[t]
\begin{center}
\resizebox{0.6\linewidth}{!}{
\begin{tabular}{c c c c}
\toprule
&  &  \multicolumn{2}{c}{Batch Norm}\\
\cmidrule{3-4}
&  & \cmark & \xmark\\
\midrule[0.7pt] \multirow{2}{*}{{Affine Transform}} & \cmark & 93.31 $\pm$ 0.17 & 92.12 $\pm$ 0.32\\
 \cmidrule{2-4}
 & \xmark & 93.18 $\pm$ 0.16 & 92.01 $\pm$ 0.21\\

\bottomrule
\end{tabular}}
\end{center}
\caption{\textbf{Ablation study.} The influence of the affine transform and batch normalization for classification on the CIFAR10 dataset is evaluated. The mean test accuracy and its standard deviation are computed over three random seeds.
}
\label{tab:ablation}
\vspace{-.2in}
\end{table}
    
\paragraph{Training details for GAN}
For GAN training, the parameter settings and model architectures for our method follow strictly with that in \cite{miyato2018spectral} and its official implementation for the training settings. More specifically, we use Adam ($\beta_1=0,\beta_2=0.9$) for the optimization with learning rate $0.0002$. We update the discriminator $5$ times per update of the generator. The batchsize is set to $64$. We adopt two performance measures, Inception score and FID to evaluate the images produced by the trained generators. The ConvNorm is added after every convolution layer in the discriminator of GAN.  

\section{Additional experiments and ablation study}\label{app:ablation}
In this section, we perform a more comprehensive ablation study to evaluate the influences of each additional component of ConvNorm on the tasks that we conducted in \Cref{sec:exp}. More specifically, we study the benefits of the extra convolutional affine transform that we introduced in \Cref{subsec:extra}, as well as 
an inclusion of a BatchNorm layer right after the ConvNorm. 



\paragraph{Fast training and better generalization.}
In~\Cref{fig:compare-precond}, we show that ConvNorm accelerates convergence and achieve better generalization performance with or without BatchNorm when regularizations such as data augmentation, weight decay, and learning rate decay are dropped during training. \Cref{fig:ablation_bn} shows that when these standard regularization techniques are added, fast convergence of ConvNorm can still be observed (see the blue curve with circles and the yellow curve with triangles). 


\begin{table}[t]
\begin{center}
\resizebox{0.8\linewidth}{!}{
\begin{tabular}{c c c c c}
\toprule
&  \multicolumn{4}{c}{Label Noise Ratio}\\
 \cmidrule{2-5}
& 20\% & 40\% & 60\% & 80\%\\
\midrule[0.7pt] 
ConvNorm + BN & 88.94 $\pm$ 0.36 & 85.88 $\pm$ 0.26 & 79.54 $\pm$ 0.73 & 69.26 $\pm$ 0.59 \\
ConvNorm & 87.75 $\pm$ 0.13 & 84.16 $\pm$ 0.71 & 77.48 $\pm$ 0.26 & 54.11 $\pm$ 2.65\\
BN & 86.98 $\pm$ 0.12 & 81.88 $\pm$ 0.29 & 74.14 $\pm$ 0.56 & 53.82 $\pm$ 1.04\\
Vanilla & 85.94 $\pm$ 0.25 & 82.11 $\pm$ 0.52 & 76.75 $\pm$ 0.20 & 57.20 $\pm$ 0.71\\

\bottomrule
\end{tabular}}
\end{center}
\caption{\textbf{Adding ConvNorm and BatchNorm together makes a network more robust to label noise.} The influence of BatchNorm and ConvNorm for label noise on the CIFAR-10 dataset is evaluated. The mean test accuracy and its standard deviation are computed over three random seeds.}
\label{tab:abl_ln}
\end{table}

\begin{table}[t]
\begin{center}
\resizebox{0.8\linewidth}{!}{
\begin{tabular}{c c c c c}
\toprule
&  \multicolumn{4}{c}{Subset Percent}\\
 \cmidrule{2-5}
& 10\% & 30\% & 50\% & 70\%\\
\midrule[0.7pt] 
ConvNorm + BN & 77.96 $\pm$ 0.11 & 87.66 $\pm$ 0.23 & 90.49 $\pm$ 0.17 & 90.71 $\pm$ 0.15 \\
ConvNorm & 69.23 $\pm$ 0.94 & 83.93 $\pm$ 0.34 & 87.83 $\pm$ 0.21 & 89.85 $\pm$ 0.10\\
BN & 67.10 $\pm$ 2.59 & 84.24 $\pm$ 0.51 & 88.74 $\pm$ 0.73 & 90.41 $\pm$ 0.33\\
Vanilla & 67.56 $\pm$ 0.50 & 81.98 $\pm$ 0.78 & 86.57 $\pm$ 0.35 & 87.61 $\pm$ 0.86\\

\bottomrule
\end{tabular}}
\end{center}
\caption{\textbf{Adding ConvNorm and BatchNorm together helps improve data efficiency} The influence of BatchNorm and ConvNorm for data scarcity on the CIFAR-10 dataset is evaluated. The mean test accuracy and its standard deviation are computed over three random seeds.}
\label{tab:abl_ds}
\end{table}

\paragraph{Robustness against label noise and data scarcity.} In \Cref{fig:data_scar_noise_label}, we show that adding a BatchNorm layer after the ConvNorm can further boost the performance against label noise and data scarcity compared with combining other baseline methods with BatchNorm. 

Here, to better understand the influence of each component, we study the effects of ConvNorm and BatchNorm separately. When we only use the ConvNorm without BatchNorm, from \Cref{tab:abl_ln} and \Cref{tab:abl_ds} we observe that in comparison to vanilla settings ConvNorm improves the performance against label noise and data scarcity for the most cases. In contrast, when only the BatchNorm is adopted, the performance downgrades that it improves upon the vanilla setting in some cases. Additionally, we notice that when we add $80\%$ of label noise to the training data, combining ConvNorm and BatchNorm together provides the best performance while using anyone alone would result in worse performance.




\paragraph{Comparasion with Cayley Transfrom \cite{trockman2021orthogonalizing}.} As mentioned in \Cref{sec:intro}, a very recent work \cite{trockman2021orthogonalizing} shares some common ideas with our work in terms of exploring convolutional structures in the Fourier domain. We note that the major difference between \cite{trockman2021orthogonalizing} and our work lies in the trade-off between the degree of orthogonality enforced and the associated computational burden. As shown in \Cref{tab:run_time}, we empirically compare the run time for training one epoch of CIFAR-10 dataset on a ResNet18 backbone using different methods. We observe that both our method and \cite{trockman2021orthogonalizing} requires more time to train compared with the vanilla network. But since our ConvNorm explores channel-wise orthogonalization instead of layer-wise as done in \cite{trockman2021orthogonalizing}, ConvNorm achieves faster training and \cite{trockman2021orthogonalizing} achieves more strict orthogonalization compared to each other. Also, we note that in terms of scalability, our ConvNorm could be adapted in larger networks such as ResNet50 and ResNet152, while the same experiments could not be carried on for \cite{trockman2021orthogonalizing} due to the limitation of our computational resources. Another important factor for comparison is the adversarial robustness. Based on our preliminary results, we found that ConvNorm has accuracy $46.12$ under PGD-10 attack, which outperforms the result of Cayley transform $38.35$ under the same attack. But we note that since the experiment settings in \cite{trockman2021orthogonalizing} are very different with ours, this comparison is not entirely fair as we have not done a comprehensive tuning for the Cayley transform method. We conjecture that with appropriate parameters and settings, the Cayley transform method could achieve on-par or even better results than ConvNorm since the more strict orthogonality enforced. 

\begin{table}[t]
\begin{center}
\resizebox{0.8\linewidth}{!}{
\begin{tabular}{c c c c}
\toprule

& Vanilla & ConvNorm & Cayley transfrom \\
\midrule[0.7pt] 
Training time (epoch) & 21s & 60s & 182s\\

\bottomrule
\end{tabular}}
\end{center}
\caption{\textbf{Training time per epoch for different methods} The average training time for one epoch of different weight normalization methods is evaluated. Experiments are conducted CIFAR-10 dataset with a ResNet18 backbone.}
\label{tab:run_time}
\end{table}

\paragraph{Layer-wise condition number.} In \Cref{fig:Lipschitz}, we have shown that ConvNorm could improve the channel-wise condition number and layer-wise spectral norm. In this section, we empirically compare the layer-wise condition number of different normalization methods. We note that we use the method described in \cite{sedghi2018singular} to estimate the singular values and condition numbers of the actual convolution operators from each layer, not the weight matrix. Here, we define a metric $\rho$ to quantify the average ratio of the condition number of the vanilla method 
and other methods
\begin{align*}
    \rho := \frac{1}{L} \sum_{l=1}^L \frac{\text{Condition number (Vanilla)}}{\text{Condition number (Method}_j})
\end{align*}
for characterizing the improvement upon the vanilla method (the larger, the better). From \Cref{tab:layer_con}, we observe that our ConvNorm shows the best result in terms of improvement of layer-wise condition number as compared with other methods. We note that we did not compare with Cayley transform \cite{trockman2021orthogonalizing} because it inherently enforces more strict orthogonality than our ConvNorm based on their experiments, so we conjecture that Cayley transform could have better condition number than our ConvNorm.

\begin{table}[t]
\begin{center}
\resizebox{0.5\linewidth}{!}{
\begin{tabular}{c c c c c}
\toprule

& SN & ONI & OCNN & ConvNorm \\
\midrule[0.7pt] 
$\rho$ & 2.724 & 0.001 & 2.288 & 3.332\\

\bottomrule
\end{tabular}}
\end{center}
\caption{\textbf{Average layer-wise condition number ratio of vanilla method on top of other methods} The experiments are conducted on natural settings with the same set of hyperparameter of \Cref{table:adv_training}.}
\label{tab:layer_con}
\end{table}

\end{document}